\documentclass[11pt,reqno]{article}
\usepackage{natbib}
\usepackage{fullpage,graphicx,xcolor}
\usepackage[colorlinks,citecolor=blue,urlcolor=cyan, linkcolor=magenta, breaklinks]{hyperref}

\usepackage{amssymb}
\usepackage{amsthm}
\usepackage{amsmath}
\usepackage{mathtools}
\mathtoolsset{showonlyrefs}
\usepackage{xparse}
\usepackage{booktabs}
\usepackage{pgf}
\usepackage{authblk}

\usepackage[utf8]{inputenc}
\usepackage[T1]{fontenc}
\usepackage{newtxtext}
\usepackage{booktabs}       
\usepackage{microtype}      

\usepackage{etoolbox}
\usepackage[enable]{easy-todo}

\usepackage{graphicx} 
\usepackage{float}
\usepackage{subcaption}
\usepackage{multirow}

\usepackage{enumitem}
\usepackage[normalem]{ulem} 

\usepackage{algorithm}

\usepackage{algpseudocode}

\let \oldsection \section
\renewcommand{\section}{\vspace{3ex plus 1ex}\oldsection}

\usepackage{empheq}
\usepackage{enumitem}
\usepackage{bbm}
\usepackage{relsize}
\usepackage{dsfont}

\newtheorem{theorem}{Theorem}[section]

\newtheorem{problem}[theorem]{Problem}

\newtheorem{lemma}[theorem]{Lemma}

\newtheorem{corollary}[theorem]{Corollary}
\newtheorem{assumption}[theorem]{Assumption}

\newcommand{\N}{\mathbb{N}}

\newcommand{\R}{\mathbb{R}}

\usepackage{xspace}

\newcommand\cK{{\ensuremath{\mathcal{K}}}\xspace}

\newcommand\cP{{\ensuremath{\mathcal{P}}}\xspace}

\newcommand\cS{{\ensuremath{\mathcal{S}}}\xspace}

\newcommand{\BEAS}{\begin{eqnarray*}}
\newcommand{\EEAS}{\end{eqnarray*}}
\newcommand{\BEA}{\begin{eqnarray}}
\newcommand{\EEA}{\end{eqnarray}}
\newcommand{\BEQ}{\begin{equation}}
\newcommand{\EEQ}{\end{equation}}
\newcommand{\BIT}{\begin{itemize}}
\newcommand{\EIT}{\end{itemize}}
\newcommand{\BNUM}{\begin{enumerate}}
\newcommand{\ENUM}{\end{enumerate}}

\newcommand{\BA}{\begin{array}}
\newcommand{\EA}{\end{array}}

\newcommand{\supp}{\mbox{\textrm{supp}}}

\newcommand{\argmin}{\mathop{\rm argmin}}

\newcommand{\argmax}{\mathop{\rm argmax}}

\title{Efficient Online-Bandit Strategies for Minimax Learning Problems}
\author[1,2]{Christophe Roux}
\author[1,2]{Elias Wirth}
\author[1,2]{Sebastian Pokutta}
\author[1,2]{Thomas Kerdreux}
\affil[1]{Technische Universit{\"a}t Berlin, Germany}
\affil[2]{Zuse Institute, Berlin, Germany}

\date{\today}

\begin{document}
\maketitle

\begin{abstract}
\noindent Several learning problems involve solving min-max problems, \textit{e.g.}, empirical distributional robust learning \citep{namkoong2016stochastic,curi2020adaptive} or learning with non-standard aggregated losses \citep{shalev2016minimizing,fan2017learning}.
More specifically, these problems are convex-linear problems where the minimization is carried out over the model parameters $w\in\mathcal{W}$ and the maximization over the empirical distribution $p\in\mathcal{K}$ of the training set indexes, where $\mathcal{K}$ is the simplex or a subset of it.
To design efficient methods, we let an online learning algorithm play against a (combinatorial) bandit algorithm.
We argue that the efficiency of such approaches critically depends on the structure of $\mathcal{K}$ and propose two properties of $\mathcal{K}$ that facilitate designing efficient algorithms.
We focus on a specific family of sets $\mathcal{S}_{n,k}$ encompassing various learning applications and provide high-probability convergence guarantees to the minimax values.
\end{abstract}

\section{Introduction}
Let $\mathcal{D}$ be a data set of $n$ i.i.d. samples from an unknown joint distribution $\mu$ with elements $(x_i,y_i)\in\mathcal{X}\times\mathcal{Y}$. We assume $\mathcal{X}\subset\mathbb{R}^d$ with $d$ and $n$ potentially large.
Further, consider a parametric family of models $(h_{w})_{w\in\mathcal{W}}$ and a loss function $\ell:\mathcal{W}\times\mathcal{X}\times\mathcal{Y}\rightarrow [0,1]$ measuring the difference between $y_i$ and the model's prediction $h_{w}(x_i)$.
Several learning problems require solving the following minimax problem
\begin{equation}\label{eq:min_max_learning}\tag{OPT}
\underset{w\in\mathcal{W}}{\text{min }} ~ \underset{p\in\mathcal{K}}{\text{max }} \langle L(w); p\rangle,
\end{equation}
where $L(w)=(\ell(h_{w}(x_1),y_1), \ldots, \ell(h_{w}(x_n),y_n))^T$ and $\mathcal{K}$ is a subset of the $n$-dimensional probability simplex $\mathcal{S}_n$.
In this paper, we design iterative methods to solve \eqref{eq:min_max_learning} that adaptively sample a mini-batch of $\mathcal{D}$ at each iteration.
For convex losses $L(\cdot)$, we provide high-probability convergence results to the optimal value of \eqref{eq:min_max_learning} for the family of sets $\mathcal{S}_{n,k}$ \eqref{eq:top_k_constraint} in Theorem \ref{th:shalev} and Theorem \ref{th:topk}.\\\\
A classical approach to designing algorithms solving \eqref{eq:min_max_learning} relies on interpreting the solution of \eqref{eq:min_max_learning} as the \textit{Nash Equilibrium} of a zero-sum game.
An \textit{online-online} strategy for such a game consists of letting an \textit{Online Learning} (OL) algorithm that seeks to maximize $\langle L(w); p\rangle$ (the $p$-player) play against an OL algorithm that seeks to minimize $\langle L(w); p\rangle$ (the $w$-player).
In learning applications where $n$ and $d$ are large, these online-online strategies become resource-intensive.
In order to solve \eqref{eq:min_max_learning}, they require to compute the loss of the $n$ data points at each round, incurring a cost scaling at least with $\mathcal{O}(nd)$.\\\\
For sets $\mathcal{K}$ with a specific structure, it is possible to design \textit{online-bandit} algorithms that consider only a subset of the data points per iteration. In such schemes, the $p$-player is chosen as a bandit algorithm that merely has access to the losses corresponding to a subset of the data, \textit{i.e.}, partial feedback.
Using bandit algorithms in order to adaptively sample data points allows for an efficient solution to \eqref{eq:min_max_learning}.
The structure of $\mathcal{K}$ first and foremost defines the learning task but it also determines whether it is possible to design dedicated efficient bandit algorithms.
We are interested in sets $\mathcal{K}$ that allow for efficient solutions and for which \eqref{eq:min_max_learning} corresponds to meaningful learning problems.
For the family of sets defined in \eqref{eq:top_k_constraint}, we provide bandit algorithms with efficient scaling of the per iteration cost w.r.t. $n$ and high-probability regret bounds in the convex-linear case.
In Appendix \ref{app:generalizing_k_set}, we introduce the \eqref{eq:alpha-set}, another family besides the \eqref{eq:top_k_constraint} for which it is also possible to design efficient online-bandit strategies.
For $k=1,\ldots,n$, consider the following subsets of the simplex $\mathcal{S}_n$,
\begin{equation}\label{eq:top_k_constraint}\tag{$k$-Set}
    \mathcal{S}_{n,k} := \Big\{ p\in\mathbb{R}^n~|~ 0\leq p_i\leq \frac{1}{k}, \sum_{i=1}^{n} p_i = 1\Big\}.
\end{equation}
Instantiating \eqref{eq:min_max_learning} with $\mathcal{K}=\mathcal{S}_{n,k}$ leads to different learning problems for different choices of $k$.
For instance, $k=1$ corresponds to learning with the aggregated max loss \citep{shalev2016minimizing} while higher values of $k$ can be interpreted as learning with the averaged top-k loss \citep{fan2017learning} or as an empirical Distributional Robust Optimization (DRO) problem \citep{curi2020adaptive}, see Section \ref{sec:related_work}.

\paragraph{Contributions.}
\begin{enumerate}
    \item We incorporate a number of different approaches into a general framework to better understand how different structures of $\mathcal{K}$ influence the possibility to efficiently solve min-max problems with online-bandit methods.
   
    \item We provide efficient algorithms with high-probability convergence guarantees for the class of sets defined by $\mathcal{S}_{n,k}$ and perform numerical experiments which illustrate their efficiency.
    
\end{enumerate}

\paragraph{Outline.}
In Section \ref{sec:banditgame}, we describe the online-bandit approach and introduce the template Algorithm \ref{algo:online_bandit_general}.
In Section \ref{sec:shalev_case} and \ref{sec:average_top_k}, we then introduce algorithms to solve \eqref{eq:min_max_learning} when $\mathcal{K}$ corresponds to the Simplex $\mathcal{S}_n$ and the \eqref{eq:top_k_constraint} $\mathcal{S}_{n,k}$ respectively. 
In Section \ref{sec:related_work}, we review related work and cover several applications of \eqref{eq:min_max_learning} in learning.
In Section \ref{sec:numerical_experiments}, we then conduct numerical experiments comparing the previously presented online-bandit algorithms to different approaches.
Then, we draw a conclusion in Section \ref{sec:conclusion}. 
In Appendix \ref{sec:regret_game}, we discuss some intricacies of designing online-bandit algorithms.
In Appendix \ref{app:additional_algorithms}, we detail some of the subroutines needed for Algorithm \ref{algo:online_bandit_k_set}.
In Appendix \ref{app:omitted_proofs}, we detail the high-probability convergence proof of Algorithm \ref{algo:online_bandit_k_set} and Algorithm  \ref{algo:online_bandit_simplex}.
In Appendix \ref{app:generalizing_k_set}, we introduce $\mathcal{K}_\alpha\,$, another example of a set for which it is possible to design \textit{efficient} online-bandit methods to solve \eqref{eq:min_max_learning}.
Finally, in Appendix \ref{sec:params}, we detail the parameters used in Section \ref{sec:numerical_experiments}.

\paragraph{Notation.}
We write $\tilde{\mathcal{O}}$ to hide logarithmic factors. 
For the minimax problem \eqref{eq:min_max_learning}, the dual gap at point $(w^\prime,p^\prime)\in\mathcal{W}\times\mathcal{K}$ is defined as
\begin{equation}
    \Delta(w^\prime, p^\prime):= \underset{p\in\mathcal{K}}{\max}~\langle L(w^\prime); p\rangle - \underset{w\in\mathcal{W}}{\min}~\langle L(w); p^\prime\rangle.\label{eq:dual_gap}\tag{Dual Gap}
\end{equation}
Let $\mathcal{P}^*\big([n]\big)$ be the set of all (non-empty) subsets of $[n]$ and $\mathcal{P}_k([n])$ the set of subsets of $[n]$ of size $k$.
We write $\mathds{1}$ for the all-ones vector (of appropriate dimension).
For $i\in[n]$, we write $\mathds{1}_{n,i}$ for the vector in $\mathbb{R}^n$ which is $1$ at coordinate $i$ and $0$ elsewhere.
For a subset $S\subset [n]$, $1_{i\in S}$ is the indicator that equals $1$ if $i\in S$ and $0$ otherwise.
We use $I$ to denote a subset of $[n]$ and $\mathcal{S}_n$ corresponds to the simplex in $\mathbb{R}^n$.
The letter $p$ stands for a probability vector in $\mathbb{R}^n$, and we write $\tilde{p}$ when the vector is not yet normalized. 
To avoid confusion, we use $q$ for probability vectors in dimensions higher than $n$.
The probability vectors $p_t$ are indexed over time and for $i\in[n]$, $p_{t,i}$ corresponds to the $i^{th}$ coordinate of $p_t$.
Similarly, $p_{t, I}$ corresponds to the coordinates $i \in I$ of $p_t$.
We write $a_t$ to denote the action of the bandit ($p$-player) at round $t$. 
For a convex compact set $\mathcal{K}$, we define by $\text{Ext}(\mathcal{K})$ the set of extreme points of $\mathcal{K}$, \textit{i.e.}, the set of points $x\in\mathcal{K}$ that cannot be expressed as a convex combination of other points in $\mathcal{K}$.
For a vector $x\in \R^n$, define its \emph{support} as $\supp(x) := \{i \in [n] \mid x_i \neq 0\}$.
For a convex differentiable function $f$, we define its Bregman divergence $D_f$ as
$D_f(x,y) := f(x) - f(y) - \langle \nabla f(y); x-y \rangle $.
We also write $\text{ReInt}(\mathcal{K})$ for its relative interior.
\section{Bandit against online learner}\label{sec:banditgame}
OL is a sequential learning framework phrased in terms of a two-player game between a learner and an adversary.
At each round $t$, the learner chooses an action $a_t\in\mathcal{K}$ and the adversary simultaneously picks a loss function $f_t$. 
The choices $f_t$ of an $adaptive$ adversary depend on the learner's previous actions $a_1, \ldots, a_{t-1}$, while and $oblivious$ adversary has to fix its whole series $f_1, \ldots, f_T$ beforehand.
The learner suffers the loss $f_t(a_t)$ and observes the full function $f_t(\cdot)$ to update its strategy in the future.
The goal of the learner is to minimize the average loss incurred, and its performance is measured by the \textit{regret},
\begin{equation}\label{eq:regret}\tag{Regret}
    R_T := \sum_{t=1}^{T} f_t(a_t) - \underset{a\in\mathcal{K}}{\text{min }} \sum_{t=1}^{T} f_t(a).
\end{equation}
Interpreting a convex-concave min-max problem as a zero-sum game between two OL algorithms with no-regret guarantees, one can prove that their time-averaged strategies converge towards an approximate saddle point, also known in this context as an approximate Nash-Equilibrium (see \textit{e.g.} \citep{wang2018acceleration}).
In order to apply this strategy to \eqref{eq:min_max_learning} one has to consider every data point at each round.
While the $p$-player optimizes a linear function on the indices of the data (memory cost $\mathcal{O}(n)$), the $w$-player has to operate on the data itself (memory cost $\mathcal{O}(nd)$). 
For large $n$ and $d$, this means that maintaining the whole dataset in memory is not feasible.
Hence the main concern is the per-iteration memory cost of the $w$-player. 
In order to make the OL approach feasible for large-scale problems, it would be preferable to evaluate only a small subset of the $n$ data points at each iteration.
This can be achieved by replacing the OL-algorithm for the $p$-player with an algorithm that works with partial information.
That way, the $w$-player is not required to evaluate the whole dataset at each iteration which alleviates the memory issues.
\begin{figure}[h]
    \centering
    \includegraphics[width=0.32\linewidth]{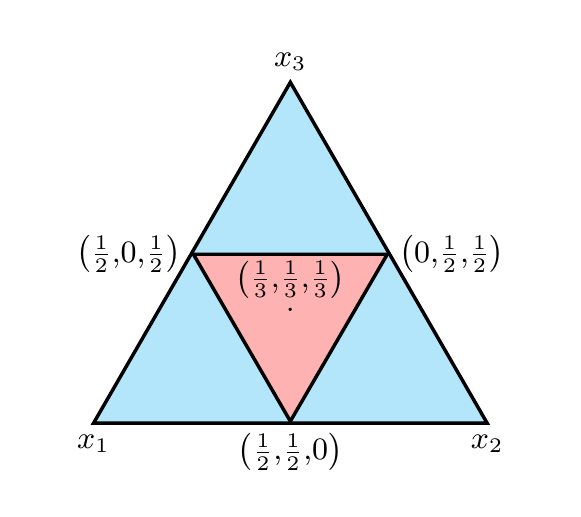}
    \includegraphics[width=0.32\linewidth]{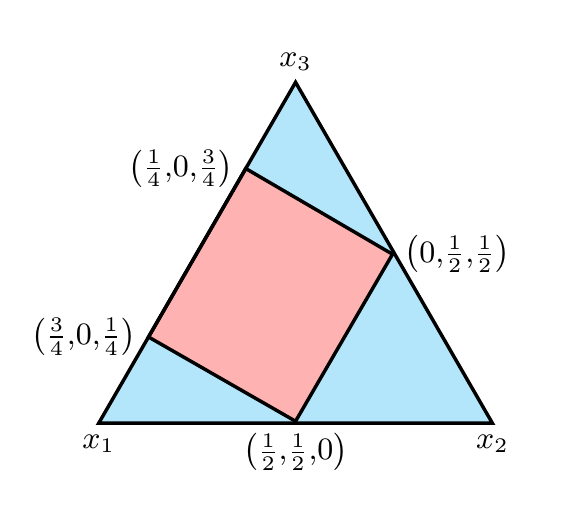}
    \includegraphics[width=0.32\linewidth]{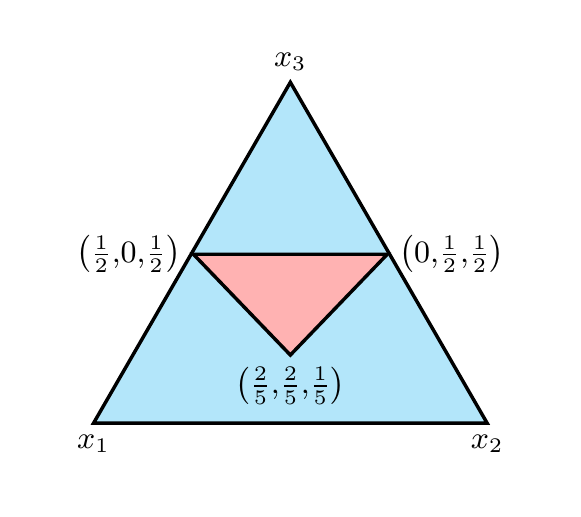}
    \caption{Three different subsets (red) of the 3-simplex $\mathcal{S}_3$ (blue). In the leftmost simplex, the red region is the \eqref{eq:top_k_constraint} $\mathcal{S}_{3,2}$ ($\mathcal{S}_{n,k}$ with $n=3$ and $k=2$). The central point corresponds to the uniform distribution or, equivalently, to $\mathcal{S}_{3,3}$.}
    \label{fig:example_k_set}
\end{figure}
\paragraph{Bandit feedback.}
The case where the learners feedback is limited to the output of $f_t(a_t)$ instead of the full function $f_t(\cdot)$ is called \textit{bandit feedback} \citep{bubeck2012regret}.
Bandit algorithms usually make up for the missing information about $f_t(\cdot)$ by constructing a statistical estimate $\hat{f}_t(\cdot)$.
Bandit algorithms rely on randomized strategies in order to achieve sublinear regret \citep{bubeck2012regret}.
Therefore regret bounds for such algorithms are usually given with high probability or in expectation, \textit{i.e.}
\begin{equation}\tag{Expected Regret}\label{eq:expected_regret}
\mathbb{E}[R_T] := \mathop{\mathbb{E}}\left[\sum_{t=1}^{T} f_t(a_t) - \underset{a\in\mathcal{K}}{\text{min }} \sum_{t=1}^{T} f_t(a)\right].
\end{equation}

In the context of \eqref{eq:min_max_learning}, the adversary of the bandit algorithm is an OL algorithm that adapts to the actions of the bandit.
This means that guaranteeing convergence requires an upper bound on the $p$-player's regret that hold against an \textit{adaptive} adversary.
These bounds can either be expressed in high-probability or in expectation and are much more challenging to obtain than pseudo-regret bounds which give meaningful guarantees only for \textit{oblivious} adversaries  \citep{abernethy2009beating,bubeck2012regret} (see Appendix \ref{sec:regret_game} for a more thorough explanation of this issue).
Regret guarantees for adaptive adversaries notably require algorithmic enhancement to better control the variance that arises from the bandits' loss estimate $\hat{L}_t$.

While bandit algorithms make it possible to solve large-scale problems by managing the memory cost, the challenge now consists of finding efficient bandit algorithms for different choices of $\mathcal{K}$.
An efficient bandit algorithm should have both an iteration cost and regret bounds that scale well with respect to the dimension $n$.
The dimension dependence of bandit algorithms is often a key challenge. For complex $\mathcal{K}$, the computational cost can be up to $\mathcal{O}(n^4)$ per iteration \citep{combes2015combinatorial}.
A computationally efficient algorithm necessitates scalable solutions for updating and sampling from its randomized strategy as well as for the projection onto $\mathcal{K}$. 
These properties crucially depend on the structure of the set $\mathcal{K}$.
We have identified two central properties of $\mathcal{K}$ which facilitate such efficient algorithms.

\paragraph{Sparsity.}
The $p$-player in \eqref{eq:min_max_learning} tries to maximize its average gain. 
Regret compares the series of actions chosen by the player to a fixed action $a^*\in \mathcal{K}$ that maximizes $\sum_{t=1}^T\langle L(w_t);a\rangle = \langle \sum_{t=1}^T L(w_t);a\rangle$. 
Finding $a^*$ is a linear optimization problem on a convex set which means that if the problem admits an optimal solution, then there is a solution that is an extreme point of $\mathcal{K}$.
Therefore, it suffices for the player to consider  $a\in \text{Ext}(\mathcal{K})$ when searching for a solution that leads to vanishing regret.
When $\mathcal{K}=\mathcal{S}_n$, the $p$-players objective coincides with the so-called \textit{Non-Stochastic Multi-Armed-Bandit} (MAB) \citep{auer2002nonstochastic}.
Here, each extreme point of $\mathcal{K}$ corresponds to a canonical direction $e_i$. 
Hence any action $a_t$ is 1-sparse \textit{i.e.}, $a_t$ has only one non-zero entry, which corresponds to evaluating only one data point per round for \eqref{eq:min_max_learning}.
For subsets of $\mathcal{K}$ some canonical directions $e_i$ might not lie in $\mathcal{K}$ anymore, meaning that there are $a\in \text{Ext}(\mathcal{K})$ with more than one non-zero entry. 
In this case, the $w$-player has to compute the loss of more than one data point at each round.
Sets $\mathcal{K}$ with sufficiently sparse extremal points allow for the design of efficient algorithms in order to control the amount of data points evaluated at each iteration.
Consider the subsets of $\mathcal{S}_3$ (the 3-simplex) in Figure \ref{fig:example_k_set} for a simple example of this issue.
While the subset on the left ($\mathcal{S}_{3,2}$) has three extreme points that are all 2-sparse, the subset on the right has two extreme points that are 2-sparse and one extreme point that requires evaluating all data points.
This issue becomes more pronounced as $n$ increases.

\paragraph{Injectivity.}
\eqref{eq:min_max_learning} is much simpler to solve using the online-bandit method when there is at most one bandit action $a\in \text{Ext}(\mathcal{K})$ associated to a subset $I\in\mathcal{P}^*\big([n]\big)$ of the data.
We formalize this as follows
\begin{assumption}[Extremal structure of $\mathcal{K}$]\label{assum:structure_K}
For the compact convex set $\mathcal{K}\subset\mathcal{S}_n$ there exists an injective function $\mathcal{T}:\text{Ext}(\mathcal{K}) \hookrightarrow \mathcal{P}^*\big([n]\big)$.
\end{assumption}
\noindent When Assumption \ref{assum:structure_K} holds, the sampling of the bandit action corresponds to a specific way of sampling indices of $[n]$ \textit{without replacement}.
If the extreme points of $\mathcal{K}$ are additionally $k$-sparse, this corresponds to sampling $k$ indices of $[n]$ without replacement.
In the center of Figure \ref{fig:example_k_set} where both $(1/4, 0 ,3/4)$ and $(3/4,0,1/4)$ correspond to sampling the data points $x_1$ and $x_2$, this convenient equivalence between choosing an action and sampling without replacement does not hold anymore.
In Appendix \ref{app:generalizing_k_set}, we describe another family of sets, for which Assumption \ref{assum:structure_K} does not apply, but each action can still be associated to a sample \textit{with replacement} of $k$ indices of $[n]$.
\paragraph{Description of Algorithm \ref{algo:online_bandit_general}.} 
Let us now describe Algorithm \ref{algo:online_bandit_general}, which has been instantiated before for some specific sets $\mathcal{K}$, \textit{e.g.}, FOL \citep{shalev2016minimizing} or Ada-CVaR \citep{curi2020adaptive}. 
First, two online learning algorithms are chosen, $w\text{-OL}_{\mathcal{W}}$ for the $w$-player and $p\text{-OL}_{\mathcal{K}}$ as a subroutine for the bandit, \textit{i.e.},the $p$-player. 
At each round, the bandit samples an action $a_t\in\text{Ext}(\mathcal{K})$ (Line \ref{line:sampling_general}).
This procedure depends on a distribution $p_t\in\mathcal{K}$ on the data-points. 
In Line \ref{line:L_estimator}, the bandit builds an estimate $\hat{L}_{t}$ of the vector $L_{I_t}$ which the $p\text{-OL}_{\mathcal{K}}$ online learning algorithm uses to update $p_{t+1}\in\mathcal{K}$. 
Simultaneously, the $w$-player observes $L_{I_t}$ and updates $w\in\mathcal{W}$.
Regret can equivalently be described in terms of gain and loss. 
In the online-bandit strategy, at each round, the $w$-player chooses $w_t$ and incurs the loss $\langle L(w_t); a_t\rangle = \frac{1}{|I_t|}\sum_{i\in I_t} \ell(h_{w_t}(x_i),y_i)$, while the bandit plays $a_t$ and gains $\langle L(w_t); a_t\rangle$. 
In the following, we will refer to $\langle L(w_t); a_t\rangle$ as the loss, whether it refers to the feedback of the $p$-player or the $w$-player.
Note that using a MAB algorithm could also be seen as using randomized coordinate descent for the $w$-player.
A significant difference is that the sampling distribution is not fixed a priori but is adapted by the bandit as the game is played.

\begin{algorithm}[h]
  \caption{Online-Bandit for \eqref{eq:min_max_learning}}
  \label{algo:online_bandit_general}
  \begin{algorithmic}[1]
  \State \textbf{Input:} $p_1\in\mathcal{K}$, $w_1\in\mathcal{W}, \mathcal{T}:\text{Ext}(\mathcal{K})\rightarrow \mathcal{P}^*\big([n]\big), p\text{-OL}_{\mathcal{K}}, w\text{-OL}_{\mathcal{W}}$
    \For{$t=1, \ldots, T $}
        \State $I_t\in\mathcal{T}\big(\text{Ext}\big(\mathcal{K}\big)\big) \sim \text{Sampler}(p_t) \hfill \vartriangleright \text{Corresponds to } a_t = \frac{1}{|I_t|}\sum_{i\in I_t} e_i.$ \label{line:sampling_general}

        \State $L_{I_t} \gets \big(\ell(h_{w_t}(x_i),y_i)\big)_{i\in I_t}\in[0,1]^{|I_t|}\hfill \vartriangleright \text{Compute loss of mini-batch.}$\label{line:semi_bandit_information}
        
        \State $\hat{L}_t \gets \text{Estimator}( L_{I_t}, I_t, p_t)\hfill \vartriangleright\text{Estimate loss vector.}$\label{line:L_estimator}
        
        \State $p_{t+1}\gets p\text{-OL}_{\mathcal{K}}(\hat{L}_t)\hfill\vartriangleright\text{Update } p\text{-player.}$ \label{line:p_player} 
        
       \State $w_{t+1}\gets w\text{-OL}_{\mathcal{W}}(L_{I_t})\hfill\vartriangleright\text{Update } w \text{-player.}$\label{line:w_player}
        
    \EndFor
    
    \State \textbf{Return:} $\mathlarger{\Bar{w}=\frac{1}{T}\sum^T_{t=1}w_t, \Bar{a}=\frac{1}{T}\sum^T_{t=1}a_t}$.
  \end{algorithmic}
\end{algorithm}

\section{Minimax learning on the simplex}\label{sec:shalev_case}
Consider \eqref{eq:min_max_learning} with $\mathcal{K}=\mathcal{S}_n$, \textit{i.e.},the n-dimensional probability simplex.
The simplex $\mathcal{S}_n$ satisfies Assumption \ref{assum:structure_K} because its extreme points $e_i$ all correspond to one of the $n$ data points.
The problem \eqref{eq:min_max_learning} now becomes
\begin{equation}\label{eq:min_max_learning_max_loss}
\underset{w\in\mathcal{W}}{\text{min }} ~ \underset{p\in\mathcal{S}_n}{\text{max }} \langle L(w); p\rangle.
\end{equation}
This problem can be solved by Algorithm  \ref{algo:online_bandit_simplex} outlined in Appendix \ref{app:online_bandit_simplex}, which is a special case of Algorithm \ref{algo:online_bandit_general} when $\mathcal{K}$ is the simplex and the $p$-player is a MAB algorithm. 
Theorem \ref{th:shalev} is a convergence guarantee of Algorithm \ref{algo:online_bandit_simplex} leveraging the high-probability guarantee of EXP.IX \citep{neu2015explore}.
This result is just a slight variation of the classical proofs which relies on two OL algorithms with full information, \textit{e.g.} \citep{wang2018acceleration}.
It can also be seen as a slight improvement on \citep[Theorem 1.]{shalev2016minimizing} in the case of convex learners, as it does not rely on the separability assumption and uses EXP.IX instead of EXP.3P, which leads to slightly faster convergence.

\begin{theorem}\label{th:shalev}
Let $\delta>0$.
Consider running $T$ rounds of Algorithm \ref{algo:online_bandit_simplex} with a choice of online learning algorithm for the $w$-player ensuring a worst case regret $R_T^w\leq C \sqrt{T}$ for some $C>0$.
Further fix the parameters
\[
\eta=2\gamma=\sqrt{\frac{2\log n}{nT}}.
\]
Write $a_t$ the actions of the bandit and $w_t$ that of the online $w$-player.
Then with probability $1-\delta$, we have
\begin{equation*}
\Delta(\Bar{w},\Bar{a}) \leq C\sqrt{\frac{1}{T}} + 2\sqrt{\frac{2n\log(n)}{T}} + \left(\sqrt{\frac{2n}{T\log{n}}}+\frac{1}{T}\right)\log \left(\frac{2}{\delta}\right), 
\end{equation*}
where $\Bar{a}=\frac{1}{T}\sum_{t=1}^{T}a_t$ and $\Bar{w}=\frac{1}{T}\sum_{t=1}^{T}w_t$. 
\end{theorem}
\begin{proof}
We present the proof in Appendix \ref{ssec:proof_simplex}.
\end{proof}

\section{Minimax learning on the capped simplex}\label{sec:average_top_k}
In this section, we apply the online-bandit approach to solve \eqref{eq:min_max_learning} on $\mathcal{K}=\mathcal{S}_{n,k}$ for $k>1$, which is a strict subset of the simplex.
The extreme points of $\mathcal{S}_{n,k}$ are characterized as follows,
\[
\text{Ext}(\mathcal{S}_{n,k})=\Big\{ \frac{1}{k}\sum_{i\in I}e_i~|~I\in\mathcal{P}_k([n])\Big\},\quad |\text{Ext}(\mathcal{S}_{n,k})|=\binom{n}{k}.
\]
An immediate way to solve \eqref{eq:min_max_learning} is by reformulating it as \eqref{eq:min_max_learning_max_loss} in a $\binom{n}{k}$-dimensional simplex, where each vertex corresponds to one $I\in\mathcal{P}_k([n])$.
The iteration cost of this approach scales with $\mathcal{O}(k\log(n))$ as $\binom{n}{k}\approx n^k$ and the regret bound in Theorem \ref{th:shalev} scales with $\tilde{\mathcal{O}}(n^{k/2})$. 
In order to handle the potentially exponential size of $\text{Ext}(\mathcal{S}_{n,k})$, one needs to leverage its combinatorial structure.\\\\
Since all actions are made up of the same $[n]$ data points, the feedback from each action contains partial information about many other actions. 
Each action $a\in \text{Ext}(\mathcal{S}_{n,k})$ corresponds to a set $I\in\mathcal{P}_k([n])$ of $k$ indices (it is $k$-sparse).
Hence the $p$-player's feedback is made up of $k$ individual losses $L_{t,i}$ each corresponding to one data point.
Allowing the learner to observe the loss of each index $i\in I$ individually provides more information than the previously considered bandit feedback and is called \textit{semi-bandit} feedback. 
The semi-bandit setting allows the learner to leverage the additional information that arises from the combinatorial nature of $\mathcal{P}_k([n])$.
More precisely, each index $i\in [n]$ is contained in $\binom{n-1}{k-1}$-many actions $I_t\in\mathcal{P}_k([n])$ and hence provides feedback for each of these actions.
Semi-bandit algorithms that are designed to solve such combinatorial problems are called \textit{Combinatorial Semi-Bandit} (CSB) algorithms.
Several algorithms have been introduced to solve the CSB problem that arises from $\mathcal{K}=\mathcal{S}_{n,k}$. 
This problem has been called the $k$-set problem \citep{combes2015combinatorial}, unordered slate \citep{Kale2010} or bandits with multiple plays \citep{uchiya2010,vural2019minimax}.

\paragraph{EXP4.MP.}
We use the EXP4.MP algorithm \citep{vural2019minimax} for the $p$-player, which is a variation of EXP4 \citep{auer2002nonstochastic}. 
Each iteration of EXP4.MP has a computational cost of $\mathcal{O}(n\log(n))$ and a storage cost of $\mathcal{O}(n)$, with a high-probability regret bound of $\mathcal{O}(\sqrt{knT\log (n/\delta)})$.
EXP4.MPs computational efficiency relies on the sampling algorithm \emph{DepRound} which has a computational cost and storage cost of $\mathcal{O}(n)$ per iteration. 
DepRound can sample a set of $k$ indices $I\in\mathcal{P}_k([n])$ requiring only a distribution $p_t$ over the indices of the data points instead of a distribution $q$ over $\mathcal{P}_k([n])$. This allows EXP4.MP to completely bypass handling a distribution over the exponentially large combinatorial set $\mathcal{P}_k([n])$.
DepRound was introduced in \citep{Gandhi2006} and was used in the context of CSB by \citep{uchiya2010,vural2019minimax}.

\paragraph{Description of Algorithm \ref{algo:online_bandit_k_set}.}
In Lines \ref{line:proj} to \ref{line:return_proj}, the unnormalized vector $\tilde{p}_t$ is first projected onto the simplex (Line \ref{line:vt}), the resulting probability distribution $v_t$ is then 
projected onto $\mathcal{S}_{n,k}$ and simultaneously mixed with the uniform distribution in order to control the variance of the estimators $\hat{L}_{t,i}$ and $U_{t,i}$.
The set $J_t$ stores the indices of $v_t$ that lie outside $\mathcal{S}_{n,k}$.
The projection algorithm costs $\mathcal{O}(n\log n)$ per iteration and is described in more detail in Appendix \ref{app:capping_alg}. 
Then, in Line \ref{line:sampling_k_set}, an action is sampled from $p_t$ via DepRound (see Appendix \ref{app:dep_round} for a thorough explanation).
In Line \ref{algo:expupdate}, the bandits weights $\tilde{p}_t$ are updated using the statistical loss estimators computed in Lines \ref{algo:estL} and \ref{algo:estU}.
Here, $\hat{L}_{t,i}$ is the importance weighted estimator and $U_{t,i}$ is an upper confidence bound obtained by using the maximal value $L_{t,i}=1$.

\begin{algorithm}[h]
  \caption{Online-Bandit (via EXP4.MP) with $\mathcal{S}_{n,k}$.}
  \label{algo:online_bandit_k_set}
  \begin{algorithmic}[1]
  \State\textbf{Input:}  $\mathlarger{\gamma,\eta,c>0,T,\tilde{p}_1 = \mathds{1}_n , k\in[n],w\text{-OL}_{\mathcal{W}}}$
    \For{$t=1, \ldots, T $}
        
        \Procedure{Proj}{$\tilde{p}_t, \gamma, k$}\label{line:proj}
        
        \State $\mathlarger{v_{t} \gets \frac{\tilde{p}_{t}}{\sum_{j=1}^n \tilde{p}_{t,j}}}$ \label{line:vt}

        \State Find $\alpha_t$ s.t. $\mathlarger{\frac{\alpha_t}{\sum_{v_{t,i}\geq \alpha_t}{\alpha_t} +\sum_{v_{t,i}< \alpha_t}{v_{t,i}}}=\frac{1/k - \gamma/n}{1-\gamma}:=\kappa} \hfill \vartriangleright \text{Algorithm } \ref{algo:proj_A}.$ \label{line:def_alpha}

        \State $J_t \gets \{i:v_{t,i}\geq \alpha_t\}$ \label{line:def_J_t}

        \State $v'_{t,i} \gets \alpha_t$ for $i\in J_t$\label{line:alpha1}

        \State $v'_{t,i}\gets v_{t,i}$ for $i\notin J_t$\label{line:alpha2}

        \State $\mathlarger{p_{t,i} \gets \Big((1-\gamma) \frac{\kappa}{\alpha_t} v'_{t,i} + \frac{\gamma}{n}\Big)}$ \label{line:def_p_t_i} 

        \State\Return $p_t, J_t$ \label{line:return_proj}
        \EndProcedure
   
        \State Sample $\mathlarger{I_t \sim \text{DepRound}(k, p_t)}\; ( \text{Corresponds to } a_t = \frac{1}{k}\sum_{i\in I_t} e_i ) \hfill\vartriangleright\text{Algorithm \ref{algo:depRound}.}$\label{line:sampling_k_set}
        
        \State $L_{t,i} \gets \ell(h_w(x_i),y_i) ~~\forall i\in I_t\subset [n]\hfill \vartriangleright\text{Compute losses of mini-batch of the dataset.}$ \label{line:loss_extraction}

        \State $\mathlarger{\Hat{L}_{t,i} \gets \frac{L_{t,i}}{kp_{t,i}}1_{[i\in I_t\backslash J_t]}~~\forall i\in [n]}\hfill \vartriangleright\text{Estimate loss vector.}$\label{algo:estL}

        \State $\mathlarger{U_{t,i} \gets \frac{1}{kp_{t,i}}1_{[i\in I_t\backslash J_t]}~~\forall i\in [n]} \hfill \vartriangleright\text{Estimate upper-bound term.}$\label{algo:estU}
        
        \State$\mathlarger{\tilde{p}_{t+1,i} \gets \tilde{p}_{t,i}\exp\big[\eta(\Hat{L}_{t,i}+U_{t,i}c/\sqrt{n T})\big]}~~\forall i\in[n]\hfill \vartriangleright\text{EXP-Update.}$\label{algo:expupdate}

        \State $\mathlarger{w_{t+1}\gets w\text{-OL}_{\mathcal{W}}\big((L_{t,i})_{i\in I_t}\big)}\hfill \vartriangleright w\text{-player update.}$
        
    \EndFor
    \State \textbf{Return: } $\mathlarger{\Bar{w}= \frac{1}{T}\sum^T_{t=1}w_t,\Bar{a}= \frac{1}{T}\sum^T_{t=1}a_t}$.
  \end{algorithmic}
\end{algorithm}
\noindent Finally we expand Theorem \ref{th:shalev} for $k>1$ to prove that the average iterates of Algorithm \ref{algo:online_bandit_k_set} indeed converge to an approximate Nash equilibrium.
\begin{theorem}\label{th:topk}
Let $\delta>0$. Consider running $T\geq \max \{\log(n/\delta), n\log(n/k)/k \}$ rounds of Algorithm \ref{algo:online_bandit_k_set} with a choice of online learning algorithm for the $w$-player ensuring a worst case regret $R_T^w\leq C \sqrt{T}$ for some $C>0$. Further fix the parameters
\begin{align*}
    \eta=\frac{k\gamma}{2n},\quad \gamma=\sqrt{\frac{n\log(n/k)}{kT}},\quad c=\sqrt{k\log(n/\delta)}.
\end{align*}
Write $a_t$ the actions of the bandit and $w_t$ those of the online $w$-player.
Then with probability $1-\delta$, we have
\begin{equation*}
    \Delta(\Bar{w},\Bar{a}) \leq C\sqrt{\frac{1}{T}} + 2\sqrt{\frac{kn}{T} \log\left(\frac{n}{\delta}\right)} +4\sqrt{\frac{kn}{T}\log \left(\frac{n}{k}\right)}+ \frac{k}{T}\log\left( \frac{n}{\delta}\right),
\end{equation*}
where $\Bar{a}=\frac{1}{T}\sum_{t=1}^{T}a_t$ and $\Bar{w}=\frac{1}{T}\sum_{t=1}^{T}w_t$.
\end{theorem}
\begin{proof}
See Appendix \ref{ssec:proof_k_set}.
\end{proof}

\section{Related work}\label{sec:related_work}
\paragraph{Adaptive bandits for matrix games.}
The repeated play approach to finding an approximate Nash-Equilibrium (\textit{i.e.}, solving a min-max problem) refers to the strategy of opposing two players equipped with online learning algorithms (the online-online setting) whose average regret converges to zero, such algorithms are often referred to as \emph{no-regret}.
These strategies have a long history initially in matrix games \citep{brown1951iterative,robinson1951iterative,blackwell1956analog,hannan1957approximation,hammond1984solving,freund1996game,freund1999adaptive}, \textit{i.e.}, minimax problems with bilinear payoff function, but have also been broadly applied to more generic minimax problems, see, \textit{e.g.}, \citep{abernethy2018faster}.
However, we are interested in the setting where one of the players is a bandit algorithm that deals with less information on its action's losses than an online algorithm would.
Such no-regret strategies, online-bandit or \textit{bandit-bandit}, are used for unknown matrix games, see, \textit{e.g.}, the multi-armed bandit strategies in \citep{auer2002nonstochastic}.

\paragraph{Beyond unknown matrix games.}
Bandit-bandit strategies for solving unknown matrix games are tied to the specificity of the games' bilinear structure with simplex constraints.
Some works \citep{hazan2011beating,clarkson2012sublinear,shalev2016minimizing} proposed learning problems with similar structures to \eqref{eq:min_max_learning} with $k=1$ that differ from matrix games. For instance, \citep{clarkson2012sublinear} consider a bilinear payoff but with non-simplicial constraints on $w$.
These approaches solve \eqref{eq:min_max_learning} with online-bandit strategies: the $p$-player (the payoff being linear w.r.t. $p$) could be a MAB \citep{auer2002nonstochastic} and the $w$-player an OL algorithm.
In unknown matrix games, the partial feedback is a modeling choice that corresponds to real-world problems \citep{myerson2013game}. 
However, in the setting of \eqref{eq:min_max_learning}, the bandit's lack of information is a necessary and intentional algorithmic design for large-scale learning problems.

\paragraph{Learning with various aggregated losses.}
For $k=1$, \eqref{eq:min_max_learning} corresponds to minimizing the largest loss incurred by a single data point, instead of the classical \textit{Empirical Risk Minimization} (ERM) objective, which corresponds to minimizing the \textit{average} loss incured by all data points in the dataset. This \emph{max-loss} was considered in \citep{clarkson2012sublinear,hazan2011beating,shalev2016minimizing}.
\citep{zhu2019robust} use the max-loss in order to provide a principled way to select unlabeled data in the context of \textit{active learning}.
In \citep{mohri2019agnostic}, the \emph{max-loss} principle is applied to subsets of the dataset instead of single data points. 
This approach is mainly motivated by Federated Learning, where the dataset consists of multiple subsets of differing size which potentially follow different distributions. 
The true data-generation distribution can be viewed as a mixture of these distributions, however the mixing coefficients are unknown. 
The max-loss is used to train a learning algorithm that is robust with respect to the change of the mixing coefficients of the subsets.
One disadvantage of the \emph{max-loss} approach is that it is sensitive to outliers.
This problem may be alleviated by considering $k>1$.
\citep{fan2017learning} introduces the concept of the \emph{aggregated loss} which generalizes different concepts of aggregating the individual losses of a dataset.
The average top-$k$ loss \textit{i.e.},the average of the $k$ largest losses \citep{fan2017learning} corresponds to \eqref{eq:min_max_learning} with $\mathcal{K}=\mathcal{S}_{n,k}$ and was introduced as an intermediary between ERM and the max-loss, limiting the influence of one data point.
\citep{fan2017learning} provides an algorithm for Support Vector Machines (SVM) with the average top-$k$ loss. 
Their algorithm empirically improves the performance relative to both the average and the maximum loss.
However, the proposed algorithm is specific to SVMs.
Furthermore it does not leverage the problem structure by adaptively sampling using bandit algorithms. 

\paragraph{Distributionally robust optimization.}
DRO can be applied as a principled way of ensuring robustness in learning problems (see, \textit{e.g.}, \citep{duchi2016statistics} and references therein).
\citep{namkoong2016stochastic} proposed an approach ensuring distributional robustness which corresponds to solving \eqref{eq:min_max_learning} for a specific family of $\mathcal{K}$, but face convergence issues, see, \textit{e.g.}, \citep[Figure 4]{ghosh2018efficient}. 
We provide a more detailed discussion of these issues in Appendix \ref{sec:regret_game}.
More closely related to our work, \citep{curi2020adaptive} propose Ada-CVaR, an online-bandit algorithm designed to solve \eqref{eq:min_max_learning} with $\mathcal{K}=\mathcal{S}_{n,k}$ when choosing $k$ such that $\alpha=k/n$ is a fixed fraction of the dataset.
In this case, one can interpret \eqref{eq:min_max_learning} as optimizing the Conditional Value at Risk (CVaR) \citep{curi2020adaptive} of the learning problem with respect to the empirical distribution. 
This means that instead of minimizing the ERM objective, \eqref{eq:min_max_learning} minimizes the loss for the $\alpha$-fraction of the data points with the largest loss.
Their approach requires sampling only one data point per round, making it potentially amenable to large-scale learning even when $k$ grows with $n$. 
For large datasets, this can be an issue since the online-bandit approach uses $k$ as the batch-size.
However, bandit algorithms that do not control the variance incur potentially unbounded variance of their loss estimators (Line 6 in ADA-CVaR). 
Furthermore, pseudo-regret guarantees are not adequate for adaptive adversaries such as in the context of games (see Appendix \ref{sec:regret_game} for a discussion of these issues). 
We were unable to extend the pseudo-regret guarantee of \citep{curi2020adaptive} to a high-probability guarantee (and hence to verify their convergence claim). This motivates our significantly different approach when solving \eqref{eq:min_max_learning}.

\paragraph{Fairness.}
There is a variety of different approaches \citep{mehrabi2019survey} trying to counteract \textit{bias} in machine learning. 
\citep{mohri2019agnostic} hint at possible applications of their agnostic learning approach to \textit{fairness}.
Indeed, \citep{williamson2019fairness} mentions the possibility of achieving fairness by controlling subgroup risk.
In other words, the performance of a learning algorithm should not vary too much between subgroups.
Minimizing the loss in the subgroup with the worst performance leads to more uniform performance and therefore reduces \textit{algorithmic bias}.
Solving \eqref{eq:min_max_learning} on $\mathcal{S}_{n,k}$ indirectly reduces bias, since minimizing CVaR enforces fairness constraints \citep{williamson2019fairness}.
The set $\mathcal{K}_\alpha$ presented in Appendix \ref{app:generalizing_k_set} provides a way to exert more granular control on the subgroup risks, which might an interesting feature when trying to ensure fairness in a principled manner.

\section{Numerical experiments}\label{sec:numerical_experiments}
\paragraph{Methodology.}
In the following numerical experiments, the $w$-player performs linear regression with the Mean-Squared Error (MSE) loss function for regression tasks and logistic regression with Cross-Entropy Loss (CEL) for classification tasks.
All experiments are repeated $5$ times with different seeds, and we show the minimal and maximal values to visualize the variance of the different approaches.
For the comparison not to depend on specific implementations, datasets, or hardware architecture, we compare the performance of the algorithms in terms of the number of data points processed.
This is necessary as the online-online approach requires processing all $n$ data points in each round, compared to $k$ in the online-bandit approach.
Since we are interested in the optimization perspective of these learning problems and not the generalization, we only compare the performance metrics on the training data.
We compare the performance of Algorithm \ref{algo:online_bandit_k_set} herein after referred to as OL-EXP.4M (as it relies on letting an OL play against EXP.4M) to solving \eqref{eq:min_max_learning} on $\mathcal{S}_{n,k}$ with mini-batch Stochastic AFL (S-AFL) \citep{mohri2019agnostic} and to replacing EXP.4M by an OL approach based on \textit{Follow the Regularized Leader} (FTRL) which we will refer to as OL-FTRL.
We set all the hyperparameters as called for in the respective theoretical results, see Appendix \ref{sec:params} for more details.
We set the uncertainty parameter for OL-EXP.4M as $\delta=0.05$ which corresponds to an error bound which holds with a probability of $0.95$.
The error bound of S-AFL holds in expectation while the error bound of OL-FTRL is deterministic.
We choose the \textit{Breast Cancer Wisconsin Dataset} \citep{uci_data} as a classification task and \textit{Boston Housing Dataset} \citep{harrison1978hedonic} as a regression task.

\paragraph{Results.}
In Figure \ref{fig:topk_cel} and Figure \ref{fig:topk_mse}, S-AFL is the slowest method, which is mainly due to two reasons. 
First, S-AFL requires sampling data points for the $w$-player and $p$-player separately at each iteration.
This means that the amount of data points that need to be sampled is double that of the other methods.
And secondly, the theoretical learning rate $\eta_p$ of the $p$-player scales in $\mathcal{O}(1/n)$ which leads to small step sizes.
For the chosen datasets, this leads to a learning rate $\eta_p$ which is at least an order of magnitude smaller than the other methods (see Table \ref{tab:topk} in Appendix \ref{sec:params}).
The speed of OL-EXP.4M and OL-FTRL are very similar in terms of the number of data points processed.
This shows that the OL-EXP.4M makes it possible to reduce the memory cost compared to online-online approaches by only evaluating $k$ instead of $n$ data points per round without losing performance and with strong theoretical guarantees.
\begin{figure}[h]
    \centering
    \includegraphics{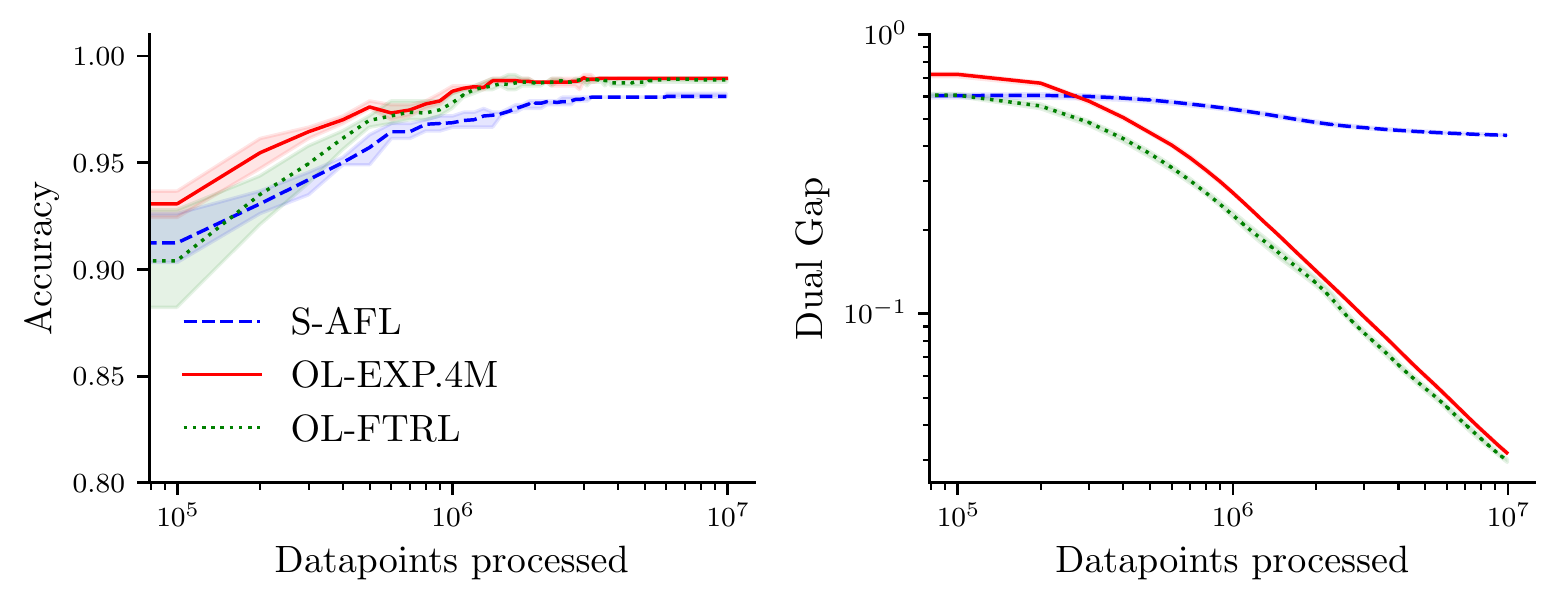}
    \caption{Classification accuracy (left) and dual gap (right) of \eqref{eq:min_max_learning} with $k=20$ on the \textit{Breast Cancer Wisconsin Dataset}.}
    \label{fig:topk_cel}
\end{figure}
\begin{figure}[h]
    \centering
    \includegraphics{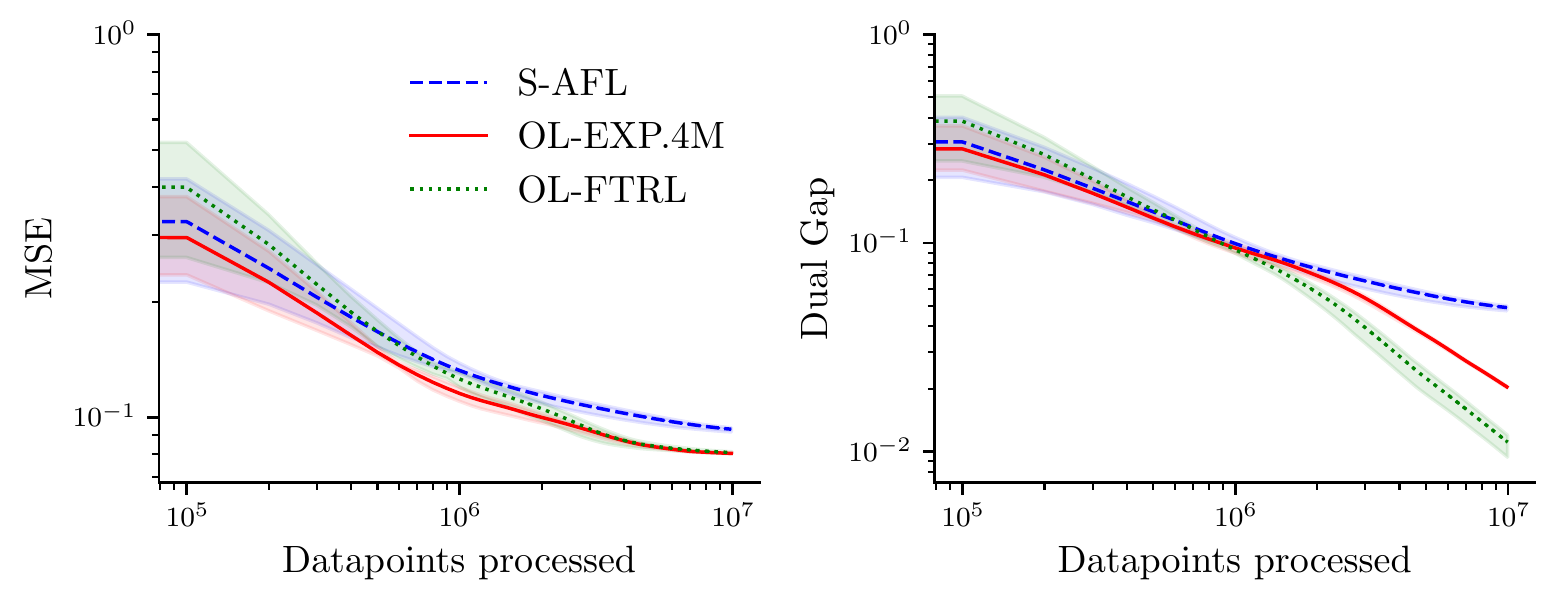}
    \caption{MSE of the $k=20$ largest losses (left) and dual gap (right) of \eqref{eq:min_max_learning} on the \textit{Boston Housing Dataset}.}
    \label{fig:topk_mse}
\end{figure}
\section{Conclusion}\label{sec:conclusion}
We presented the generic min-max problem \eqref{eq:min_max_learning} which can be solved using online-bandit strategies for large-scale learning problems when $\mathcal{K}$ satisfies some structural assumptions.
Furthermore, we provide an efficient algorithm for a family of sets $\mathcal{S}_{n,k}$ at the core of many learning applications, which satisfies these structural assumptions.
Note that the iteration batch size is constrained by the structure of $\mathcal{K}$.
In future work, we would like to design strategies providing the same strong theoretical guarantees while sampling only a fraction of the data points our setting currently requires per iteration.
\paragraph{Ackowledgements}
Research reported in this paper was partially supported through the Research Campus Modal funded by the German Federal Ministry of Education and Research (fund numbers 05M14ZAM,05M20ZBM) as well as the Deutsche Forschungsgemeinschaft (DFG) through the DFG Cluster of Excellence MATH+.
\bibliography{utils/biblio}
\bibliographystyle{abbrvnat}

\appendix

\section{Game convergence problems} \label{sec:regret_game}
In this section, we discuss some intricacies of choosing appropriate bandit algorithms to solve \eqref{eq:min_max_learning}.

\subsection{Action set}
When solving \eqref{eq:min_max_learning} with the online-bandit approach, the structure of $\mathcal{K}$ dictates which algorithms the $p$-player can use.
At each round, the $p$-player tries to find the action $a\in \mathcal{K}$, which maximizes the weighted sum of losses $\langle L(w_t);a\rangle$.
The bandit's performance is then measured in terms of regret, which compares the series of actions chosen by the player to a fixed action $a^*\in \mathcal{K}$ that maximizes $\sum_{t=1}^T\langle L(w_t);a\rangle = \langle \sum_{t=1}^T L(w_t);a\rangle$. 
Finding $a^*$ is a linear optimization problem on a compact convex set which means that an extreme point of $\mathcal{K}$ is a solution.
Hence, it suffices for the $p$-player to consider the actions $a\in \text{Ext}(\mathcal{K})$ when searching for a solution that leads to vanishing regret, \textit{i.e.}, its randomized strategy is a probability distribution over $\text{Ext}(\mathcal{K})$.
In the special case of $\mathcal{K}=\mathcal{S}_n$, this means that the $p$-player's objective is to find the index $i\in [n]$ corresponding to one data point which maximizes $\langle L;e_i\rangle$.
This coincides with the MAB problem, which is a special case of the linear bandit where the player's set is the simplex and the player can choose among $n$ discrete actions.
The randomized strategy of the MAB is a discrete probability distribution over these $n$ actions.\\\\
Now, consider the following set introduced in \citep{namkoong2016stochastic},
\[
\mathcal{U} = \{p\in\mathcal{S}_n\vert D(p\Vert\mathds{1}/n)\leq \rho \}.
\]
$D(\cdot || \cdot)$ is a divergence measure between two probability distributions such as the KL-divergence or the $\chi^2$-divergence, $\mathds{1}/n$ is the $n$-dimensional uniform distribution and $\rho$ is a parameter which controls the amount of divergence from the uniform distribution. 
Interestingly, \citep{namkoong2016stochastic} chooses MAB-algorithms for the $p$-player in order to solve \eqref{eq:min_max_learning}.
Consider an example of such an action set $\mathcal{U}$ defined by the $\chi^2$ divergence around the uniform distribution in Figure \ref{fig:set_error} on the left.
Here, $\text{Ext}(\mathcal{U})$ is the circle parametrized by $D(p\Vert\mathds{1}/n)= \rho$. 
Consequently, we have $a_i\in(0,1)$ (for $a\in \text{Ext}(\mathcal{U})$) when $\rho$ is sufficiently small, \textit{i.e}, 
none of the extreme points of $\mathcal{U}$ have zero-entries.
MAB-algorithms on the other hand are designed to choose between a discrete set of $n$ actions, which are represented by the one-sparse canonical bases $e_i$.
This means that when $\mathcal{K}=\mathcal{U}$ choosing an MAB-algorithm seems not to be an appropriate choice to solve \eqref{eq:min_max_learning}, as it considers actions outside of $\mathcal{U}$.
This issue arises from conflating the randomized strategy $p$ over $\mathcal{K}$ with the actual action $a\in \mathcal{K}$.\\\\
When $\mathcal{K} = \mathcal{S}_{n,k}$, the randomized strategy of the $p$-player is a probability distribution over the $\binom{n}{k}$ different actions in $\text{Ext}(\mathcal{S}_{n,k})$.
However, in \citep{curi2020adaptive}, the $p$-player is chosen as EXP3 (an algorithm designed for the MAB problem) operating on the $n$-dimensional probability simplex to solve \eqref{eq:min_max_learning} on $\mathcal{S}_{n,k}$.
In Figure \ref{fig:set_error}, the right image shows $\mathcal{S}_{3,2}$, which has three extreme points, each with 2 nonzero entries. 
However, any action sampled by a MAB algorithm corresponds to a canonical direction, which is one-sparse.
\begin{figure}[h]
    \centering
    \includegraphics{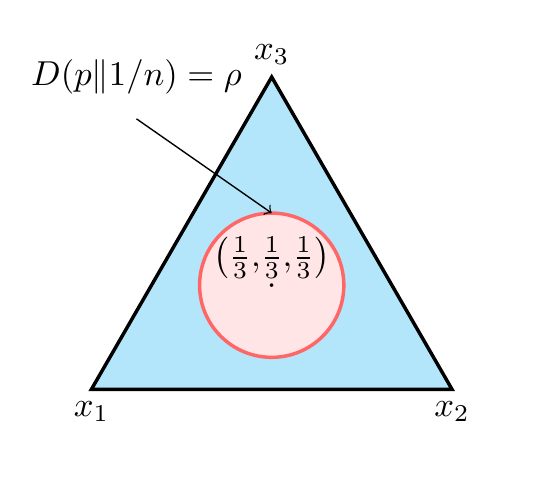}
    \includegraphics{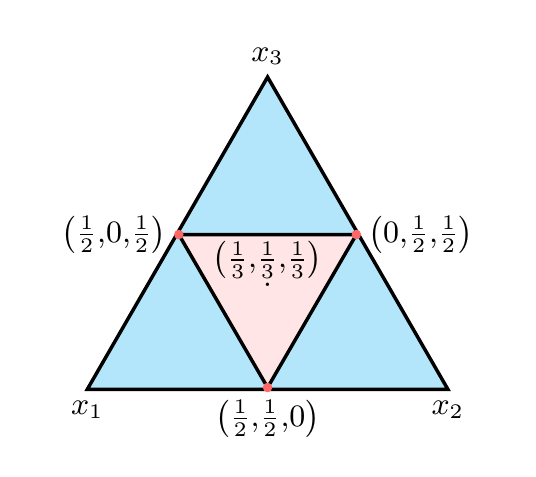}
    \caption{Two subsets (light red) of the 3-simplex $\mathcal{S}_{3}$ (blue) with the extreme points marked in red. The left subset is defined by the $\chi^2$ divergence around the uniform distribution \citep{namkoong2016stochastic} and the right subset is $\mathcal{S}_{3,2}$}
    \label{fig:set_error}
\end{figure}

\subsection{Bandit performance measures}
In Section \ref{sec:banditgame}, we introduced \eqref{eq:regret} as well as \eqref{eq:expected_regret}.
Another measure of regret is the \textit{pseudo-regret},
\begin{equation}\tag{Pseudo-regret}\label{eq:pseudo_regret}
    \tilde{R}_T := \underset{a\in\mathcal{K}}{\text{min }}\mathop{\mathbb{E}}\left[\sum_{t=1}^{T} f_t(a_t) -  \sum_{t=1}^{T} f_t(a)\right].
\end{equation}
Algebraically, the pseudo-regret differs from regret only in that the $\min$-operator is outside the expectation.  
Since $\min$ is concave, by Jensen's inequality we have $\tilde{R}_T\leq \mathbb{E}[R_T]$.
In general, an upper bound on the pseudo-regret does not lead to an upper bound on the expected regret. 
Only in the case of an oblivious adversary -  meaning that the adversary's choices of $f_t$ do not depend on the previous actions $a_1, \ldots, a_{t-1}$ of the bandit - do the expected regret and pseudo-regret coincide \citep{audibert2010regret}.
For adaptive adversaries, the pseudo-regret serves as an intermediary technical step toward obtaining the more involved high-probability or expected regret bounds \citep{kerdreux2021linear}.
Regret bounds which hold against adaptive adversaries are instrumental in proving convergence to an approximate Nash equilibrium of \eqref{eq:min_max_learning} \citep[\S 1]{bubeck2012regret}.
When two OL-algorithms play against each other, the feedback that each player receives is determined by the opposed player.
The players then use this feedback to adapt their strategy so that each player is an adaptive adversary to the other.\\\\
To clarify the issue with using algorithms for which we only know of pseudo-regret guarantees in the context of solving zero-sum games, let us consider the procedure used to prove game convergence by opposing two OL-algorithms as shown in Theorem \ref{thm:general}.
Suppose that instead of an expected regret guarantee, we only know of a guarantee in pseudo-regret, \textit{i.e.},
\begin{align*}
    \tilde{R}_T^p &:= \underset{a\in\text{Ext}(\mathcal{K})}{\text{max }}\mathbb{E}\left[\sum_{t=1}^{T} \langle L(w_t); a\rangle -  \sum_{t=1}^{T} \langle L(w_t); a_t\rangle\right] \leq \tilde{\epsilon}_p(T),
\end{align*}
with $\tilde{\epsilon}_p(T)=o(T)$.
By convexity of $w\mapsto\langle L(\cdot); a_t\rangle$, we obtain
\begin{align}
    \underset{a\in\text{Ext}(\mathcal{K})}{\text{max }}\mathbb{E}\left[\langle L(\bar{w}); a\rangle -\frac{1}{T}  \sum_{t=1}^{T} \langle L(w_t); a_t\rangle\right]\leq \frac{1}{T}\tilde{R}_T^p\leq\frac{1}{T} \tilde{\epsilon}_p(T),\label{eq:p_reg_p}
\end{align}
where $\bar{w}= \frac{1}{T}\sum_{t=1}^Tw_t$. By linearity of $a\mapsto \langle L(w_t); \cdot\rangle$, we have the regret guarantee of the $w$-player (see \eqref{eq:reg_w}),
\begin{equation*}
    \frac{1}{T}R_T^w = \frac{1}{T}\sum_{t=1}^{T} \langle L(w_t); a_t\rangle - \underset{w\in\mathcal{W}}{\text{min }}\langle L(w); \bar{a}\rangle\leq \frac{1}{T}\epsilon_w(T),
\end{equation*}
where $\bar{a}= \frac{1}{T}\sum_{t=1}^Ta_t$.
If we now add \eqref{eq:p_reg_p} and the regret guarantee of the $w$-player, we get
\begin{align*}
    \frac{1}{T}\epsilon_w(T) + \frac{1}{T}\tilde{\epsilon}_p(T) \geq  \underset{a\in\text{Ext}(\mathcal{K})}{\text{max }}\mathbb{E}\left[\langle L(\bar{w}); a\rangle - \underset{w\in\mathcal{W}}{\text{min }}\langle L(w); \bar{a}\rangle\right].
\end{align*}
By Jensen's inequality, the term on the right side is now a lower bound on the dual gap as the $\max$-function is convex, \textit{i.e.},
\begin{align*}
    \underset{a\in\text{Ext}(\mathcal{K})}{\text{max }}\mathbb{E}\left[\langle L(\bar{w}); a\rangle - \underset{w\in\mathcal{W}}{\text{min }}\langle L(w); \bar{a}\rangle\right]&\leq \mathbb{E}\left[ \underset{p\in\mathcal{K}}{\max}~\langle L(w^\prime); p\rangle - \underset{w\in\mathcal{W}}{\min}~\langle L(w); p^\prime\rangle\right]\\
    &=\mathbb{E}\left[\Delta(\bar{w},\bar{p})\right].
\end{align*}
Hence, replacing expected regret by pseudo regret has the consequence that the sum of the average regret for the $p$-player and the $w$-player does not lead to an upper bound on \eqref{eq:exp_dual_gap} anymore.
This illustrates that algorithms' pseudo-regret bounds do not lead to convergence guarantees for solving zero-sum games.\\\\
In addition, algorithms with pseudo-regret guarantees do not control the variance, which might lead to convergence issues in practice.
EXP3 for example uses the statistical estimate $\tilde{L}_i = \frac{L_i}{p_i} \mathds{1}_i$ which has unbounded variance.
The loss estimate $\tilde{L}$ is inversely proportional to the probabilities of each action and hence, samples with small probabilities will lead to exploding estimates.
The problem of high variance also comes up in other optimization techniques \citep{rockafellar2000optimization}, see, \textit{e.g.} \citep{namkoong2016stochastic} or \citep[\S Challenge for Stochastic Optimization]{curi2020adaptive}.
It has been noted that the method in \citep{namkoong2016stochastic} faces some convergence issues, see, \textit{e.g.}, \citep[Figure 4]{ghosh2018efficient}.
Furthermore, while \citep{curi2020adaptive} consider EXP3 in their theoretical analysis, they mix EXP3's distribution with the uniform distribution in their numerical experiments to stabilize the optimization process.

\section{Additional Algorithms}\label{app:additional_algorithms}

\subsection{Online-Bandit on the Simplex} \label{app:online_bandit_simplex}
Algorithm \ref{algo:online_bandit_simplex} is a special case of Algorithm \ref{algo:online_bandit_general} when $\mathcal{K}=\mathcal{S}_n$ and the $p$-player is a MAB.
\begin{algorithm}[h]
  \caption{Online-Bandit with $\mathcal{K}=\mathcal{S}_n$.} \label{algo:online_bandit_simplex}
  \begin{algorithmic}[1]
  \State \textbf{Input:} $\mathlarger{\gamma,\eta>0,T, \tilde{p}_1=\mathds{1}_n}, w\text{-OL}_{\mathcal{W}}$

  \For{$t=1, \ldots, T $}
  
  \State $\mathlarger{p_{t,i} \gets  \tilde{p}_{t,i}/\big(\sum^n_{j=1}\tilde{p}_{t,j}\big)~~\forall i \in[n]}$

  \State Sample $\mathlarger{i_t\in [n] \sim p_t}$
 
  \State $\mathlarger{a_t \gets e_{i_t}}$

  \State $\mathlarger{L_{t,i_t}\gets\ell(h_{w_t}(x_{i_t}),y_{i_t})}$
  
  \State $\mathlarger{\Hat{L}_{t}\gets \frac{L_{t,i_t}}{(p_{t,i_t}+\gamma)}\mathds{1}_{n,i_t}\in\mathbb{R}^n}$\label{algo:estimate}
  
  \State $\mathlarger{\tilde{p}_{t+1,i} \gets \tilde{p}_{t,i}\exp(-\eta\Hat{L}_{t,i})~~\forall i\in[n]}$\label{algo:MWU}

  \State $\mathlarger{w_{t+1}\gets w\text{-OL}_{\mathcal{W}}(L_{t,i_t})}$

  \EndFor
  
  \State \textbf{Return:} $\mathlarger{\Bar{w}=\frac{1}{T}\sum^T_{t=1}w_t, \Bar{a}=\frac{1}{T}\sum^T_{t=1}a_t}.$
  \end{algorithmic}
\end{algorithm}

\subsection{Projection Algorithm}\label{app:capping_alg}
In Algorithm \ref{algo:online_bandit_k_set}, Lines \ref{line:proj}-\ref{line:return_proj} correspond to the so-called PROJection procedure.
It outputs a probability vector $p_t\in \mathcal{S}_{n,k}$ (and a subset $J$ of $[n]$) given a vector $\tilde{p}_t$ in the positive orthant $\mathbb{R}_{+*}^n$ and the parameters $k\in[n]$ and $\gamma\in[0,1]$.
This procedure performs a Bregman Projection of $\tilde{p}_t$ onto $\mathcal{S}_{n,k}$ such that the resulting vector is already mixed with the uniform distribution.
We now detail Algorithm \ref{algo:proj_A}, and explain its behavior.

\begin{algorithm}[h]
  \caption{Find $\alpha_t$.}
  \label{algo:proj_A}
  
  \begin{algorithmic}[1]
  
  \State \textbf{Input:} $\tilde{p} \in \R^n_{+*}$, $k\in [n]$, $\gamma > 0$
  
  \State $\kappa \gets \frac{1/k-\gamma/n}{1-\gamma}$
  
  \State $v \gets \frac{\tilde{p}}{\|\tilde{p}\|_1}$ \label{line:p}
  
  \State $v \gets$ Sort $(v_1, \ldots, v_n)^T$ in a descending order \label{line:sort}
  \State $i \gets 1$
  
  \State $v^\prime \gets (v_1, \ldots, v_n)^T $
  
  \While {$\max_{r=1,\cdots,n}(v^\prime_r) > \kappa$} \label{line:while_max} 
  
    \State $v^\prime \gets (v_1, \ldots, v_n)^T$ \label{line:v_prime_algo_4}
    
    \State $v_j' \gets \kappa$ for $j = 1, \ldots, i$
    
    \State $v_j' \gets (1 - \kappa i) \frac{v_j'}{\sum_{\ell = i + 1} v_{\ell}'}$ for $j = i + 1, \ldots, n$ \label{line:v_prime_algo_4_bis}
    
    \State $i \gets i + 1$ \label{line:increment}
    
  \EndWhile
  
  \State $\alpha \gets \frac{\kappa }{1-(i-1)\kappa}\sum_{j=i}^n v_j'$
  
  \State \textbf{Return: } $\alpha$
  
  \end{algorithmic}
\end{algorithm}

Let us introduce a slight variation of $\mathcal{S}_{n,k}$.
Namely, for $\kappa\in [0,1]$ we write
\begin{equation}
    \mathcal{S}_{n,\frac{1}{\kappa}} := \Big\{ q\in\mathbb{R}^n~|~ 0\leq q_i\leq \kappa, \sum_{i=1}^{n} q_i = 1\Big\}.
\end{equation}
If one simply projects $\tilde{p}_t$ onto $\mathcal{S}_{n,k}$ and then performs the convex combination of the resulting probability vector $v'_t$ with the uniform distribution $p_{t,i} = (1-\gamma) v'_{t,i} + \frac{\gamma}{n}$, then the resulting vector does not lie on $\text{Ext}(\mathcal{S}_{n,k})$ but on $\text{Ext}(\mathcal{S}_{n,\frac{1}{\kappa}})$, where $\kappa=\frac{1/k-\gamma/n}{1-\gamma}$.
Hence, we first project onto $\mathcal{S}_{n,\frac{1}{\kappa}}$, in order to end up with $p_t\in \text{Ext}(\mathcal{S}_{n,k})$ after mixing with the uniform distribution.
Algorithm \ref{algo:proj_A} searches the scalar $\alpha_t$ required to compute the Bregman projection w.r.t. the Kullback-Leibler divergence of a probability distribution onto $\mathcal{S}_{n,\frac{1}{\kappa}}$.
Algorithm \ref{algo:proj_A} appears in \citep{herbster2001tracking} or \citep[Algorithm 4]{warmuth2008randomized,vural2019minimax}.
For completeness, we now recall the necessary conditions satisfied by the Bregman projection onto $\mathcal{S}_{n,\frac{1}{\kappa}}$ in Lemma \ref{lem:projection_S_n_kappa} \citep{herbster2001tracking}. 
We then prove in Lemma \ref{lem:termination_projection} that Algorithm \ref{algo:proj_A} indeed terminates and characterize its outputs.
\begin{lemma}[Projection on $\mathcal{S}_{n,\frac{1}{\kappa}}$]\label{lem:projection_S_n_kappa}
Let $p\in\text{ReInt}(\mathcal{S}_n)$ and $\kappa\in ]0,1]$. Then there is a unique $p^*\in\mathcal{S}_{n,\frac{1}{\kappa}}$ such that
\begin{equation}\label{eq:projection_S_n_kappa}
    p^* \in \argmin_{q\in\mathcal{S}_{n,\frac{1}{\kappa}}} D_{F}(q, p) = \sum_{i=1}^{n} q_i \ln\Big(\frac{q_i}{p_i}\Big),
\end{equation}
where $F(q)= \sum_{i=1}^{n}{q_i \ln (q_i) - q_i}$ and $D_{F}$ the Bregman divergence of $F$, \textit{i.e.}, the Kullback-Leibler divergence.
We say that $p^*$ is the projection of $p$ onto $\mathcal{S}_{n,\frac{1}{\kappa}}$ w.r.t. the KL divergence.
Besides, write $J\triangleq\{i\in[n]\text{ s.t. }p_i^*=\kappa\}$. The following properties are true
\begin{enumerate}[label=(\alph*)]
    \item For all $i\notin J$, $p_i^* = \frac{1- |J|\kappa}{\sum_{j\notin J} p_j} p_i$. \label{itm:first_claim}

    \item If $(p_i)$ is non-decreasing, there exists $r>0$ such that
    $p^* = (\kappa, \cdots, \kappa, r p_{|J|+1}, \cdots, r p_{n})$. \label{itm:second_claim}
    
    \item Let $m\in[n]$ such that $r_0 = \frac{1-m\kappa}{1-\sum_{i=1}^{m}{p_i}}>0$ and $r_1 = \frac{1-(m+1)\kappa}{1-\sum_{i=1}^{m+1}{p_i}}>0$.
    Consider
    \[
    v^\prime = (\kappa, \cdots, \kappa, r_0 p_{m+1}, \cdots, r_0 p_{n}) \in \mathcal{S}_{n} \text{ and } v^{\prime\prime} = (\kappa, \cdots, \kappa, r_1 p_{m+2}, \cdots, r_1 p_{n}) \in \mathcal{S}_{n}.
    \]
    If $(p_i)$ is non-decreasing, we have
    \[
    D_{F}(v^\prime, p) \leq D_{F}(v^{\prime\prime}, p).
    \]
    \label{itm:third_claim}
\end{enumerate}
\end{lemma}

\begin{proof}
There is a unique solution to \eqref{eq:projection_S_n_kappa} for $D_F$ when $F$ is strictly convex and differentiable \citep{bregman1967relaxation}.
Items \ref{itm:first_claim}-\ref{itm:third_claim} are closely related to Claims $1-3$ in \citep{herbster2001tracking}.
We slightly adapt them and repeat the proof for completeness.
Recall that $\mathcal{S}_{n,\frac{1}{\kappa}}$ is written as
\[
\mathcal{S}_{n,\frac{1}{\kappa}}=\Big\{ q\in\mathbb{R}^n~|~0 \leq q_i \leq \kappa ~~\forall i\in[n] \text{ and } \sum_{i=1}^{n}{q_i}=1 \Big\}.
\]
Let us prove \ref{itm:first_claim}. 
Consider the optimization problem \eqref{eq:projection_S_n_kappa}.
Write $\mu_1\in\mathbb{R}^n$ (respectively $\mu_2\in\mathbb{R}^n$) the dual variables associated to the inequality constraints $0 \leq q_i$ (respectively $q_i \leq \kappa$) and $\nu\in\mathbb{R}$ the dual variable associated to the equality constraint $\sum_{i=1}^{N}{q_i}=1$. 
Write the Lagrangian of \eqref{eq:projection_S_n_kappa} as
\[
\mathcal{L}(q, \mu_1, \mu_2, \nu) = D_{F}(q, p) - \langle \mu_1; q\rangle + \langle \mu_2 ; q - \kappa\mathds{1}\rangle + \nu \big(1 - \langle \mathds{1}; q\rangle\big).
\]
Write $\mu_1^*,\mu_2^*,\nu^*$ the KKT multipliers associated to the solution $p^*$ of \eqref{eq:projection_S_n_kappa}.
The first-order condition means that $\nabla_q \mathcal{L}(p^*,\mu_1^*, \mu_2^*, \nu^*)=0$. Hence, for any $i\in[n]$, we have
\[
\ln\Big(\frac{p_i^*}{p_i}\Big) + 1 - \mu_{1,i}^* + \mu_{2,i}^* - \nu^*=0.
\]
Hence, $p_i^* = p_i e^{\nu^* + \mu_{1,i}^* - \mu_{2,i}^* - 1}$ for all $i\in[n]$.
Then the complementary slackness condition means that $\mu_{1,i}^*p^*_i=0$ and $\mu^*_{2,i}(p^*_i - \kappa)=0$ for all $i\in [n]$.
Hence, by definition for $i\notin J$, $p_i^*\neq \kappa$ so that complementary slackness implies $\mu^*_{2,i}=0$. 
Also, by contradiction, if there exists $i\in[n]$ s.t. $p_i^*=0$, then $p_i=0$ which contradicts the fact that $p_i\in\text{ReInt}(\mathcal{S}_n)$. Then by complementary slackness, we have $\mu_{1,i}^*=0$ for all $i\in[n]$.
Hence, for all $i\notin J$ ($p_i^*\neq \kappa$), we have $p_i^*=p_i e^{\nu^*-1}$.
Since, by primal feasibility $\sum_{i=1}^{n} p_i^* = 1$, we have $e^{v^*-1}\sum_{i\notin J} p_i + |J| \kappa = 1$ which
implies \ref{itm:first_claim}.\\\\
Let us prove \ref{itm:second_claim}. Assume $(p_i)$ is non-decreasing.
From \ref{itm:first_claim}, we know that $p_i^*=\kappa$ for $i\in J$ and $p_i^* = \frac{1- |J|\kappa}{\sum_{j\notin J} p_j} p_i$ for $i\notin J$.
If $\kappa \geq \frac{1- |J|\kappa}{\sum_{j\notin J} p_j} p_i$ for all $i\notin J$, then \ref{itm:second_claim} is satisfied. 
Let us assume by contradiction that it is not the case. Then, there exists $i_0\notin J$ s.t. $\kappa < \frac{1- |J|\kappa}{\sum_{j\notin J} p_j} p_{i_0}$. 
Also, because $(p_i^*)$ is decreasing, we have $j_0\in J$ (i.e., $p^*_{j_0}=\kappa$) s.t. $i_0 < j_0$.
Write $p^\prime\in \mathbb{R}^n$ s.t. $p^\prime_{i_0} = p^*_{j_0}$, $p^\prime_{j_0} = p^*_{i_0}$ and otherwise $p^\prime_{i} = p^*_{i}$.
We have
\begin{eqnarray*}
    D_{F}(p^\prime, p) - D_{F}(p^*, p) & = & p^\prime_{i_0} \ln\Big(\frac{p^\prime_{i_0}}{p_{i_0}}\Big) + p^\prime_{j_0} \ln\Big(\frac{p^\prime_{j_0}}{p_{j_0}}\Big) - p^*_{i_0} \ln\Big(\frac{p^*_{i_0}}{p_{i_0}}\Big) - p^*_{i_0} \ln\Big(\frac{p^*_{j_0}}{p_{j_0}}\Big)\\
    & = & -p^\prime_{i_0} \ln(p_{i_0}) - p^\prime_{j_0} \ln(p_{j_0}) + p^*_{i_0} \ln(p_{i_0}) + p^*_{i_0} \ln(p_{j_0})\\
    & = & \ln\Big(\frac{p_{j_0}}{p_{i_0}}\Big) \big[p_{j_0}^* - p_{i_0}^*\big] = \underbrace{\ln\Big(\frac{p_{j_0}}{p_{i_0}}\Big)}_{\geq 0} \underbrace{\Big[\kappa - \frac{1- |J|\kappa}{\sum_{j\notin J} p_j} p_{i_0}\Big]}_{< 0}.
\end{eqnarray*}

Hence, we have $D_{F}(p^\prime, p) \leq D_{F}(p^*, p)$ which contradicts the uniqueness of $p^*$ as a solution to \eqref{eq:projection_S_n_kappa}.
This proves \ref{itm:second_claim}.\\\\
Let us finally prove \ref{itm:third_claim}.
Consider the following optimization problem .
\[
\underset{\substack{q\in\mathcal{S}_{n}\\ q_i=\kappa ~~\forall i\in[m]}}{\text{minimize }} D_{F}(q,p)=\sum_{i=1}^{n} q_i \ln\Big(\frac{q_i}{p_i}\Big).
\]
Deriving the KKT conditions we obtain that $v^\prime$ is the solution. Since $v^{\prime\prime}$ is feasible for that problem, we have $D_{F}(v^\prime, p)\leq D_{F}(v^{\prime\prime}, p)$.
\end{proof}
\noindent With Lemma \ref{lem:termination_projection}, we now prove that the $v^\prime$ (up to a reordering of the coefficients) constructed in Lines \ref{line:v_prime_algo_4}-\ref{line:v_prime_algo_4_bis} of Algorithm \ref{algo:proj_A} is the Bregman Projection of $\tilde{p}\ /\|\tilde{p}\|_1$ on $\mathcal{S}_{n, 1/\kappa}$.
Similarly, the $v^\prime$ in Lines \ref{line:alpha1} and \ref{line:alpha2} of Algorithm \ref{algo:online_bandit_k_set} is also the Bregman projection of $v_t = \tilde{p}_t/\|\tilde{p}_t\|_1$ (Line \ref{line:def_alpha} of Algorithm \ref{algo:online_bandit_k_set}) on $\mathcal{S}_{n,\frac{1}{\kappa}}$. 
Finally, note that we assume $\tilde{p}_i>0$ to apply Lemma \ref{lem:projection_S_n_kappa}.
In Algorithm \ref{algo:online_bandit_k_set}, this is always satisfied if the initial condition is properly chosen, \textit{i.e.}, if $\tilde{p}_{0,i}>0$ for all $i\in [n]$.
We formalize this discussion in Lemma \ref{lem:termination_projection}.

\begin{lemma}\label{lem:termination_projection}
Consider $\tilde{p}\in\mathbb{R}_{+,*}^n$ and $v=\tilde{p}/\|\tilde{p}\|_1$.
Then, Algorithm \ref{algo:proj_A} terminates and the output $\alpha$ satisfies
\begin{equation}\label{eq:alpha_equation}
    \frac{\alpha}{\sum_{v_i>\alpha}{\alpha} + \sum_{v_i\leq\alpha}{v_i}} = \frac{1/k - \gamma/n}{1-\gamma} = \kappa.
\end{equation}
Also, $v^\prime_t$ as defined in Lines \ref{line:alpha1} and \ref{line:alpha2} of Algorithm \ref{algo:online_bandit_k_set} is the Bregman projection (w.r.t. the KL divergence) of $v_t$ (Line \ref{line:vt} of Algorithm \ref{algo:online_bandit_k_set}) on $\mathcal{S}_{n,\frac{1}{\kappa}}$.
Finally, $p_t$ as defined in Line \ref{line:def_p_t_i} of Algorithm \ref{algo:online_bandit_k_set} belongs to $\mathcal{S}_{n,k}$.
\end{lemma}

\begin{proof}
Assume, without loss of generality, that $(v_{i})$ is non-decreasing.
Then, with Lemma \ref{lem:projection_S_n_kappa} \ref{itm:first_claim} and \ref{itm:second_claim}, the Bregman $v^*$ projection \eqref{eq:projection_S_n_kappa} of $v$ onto $\mathcal{S}_{n,\frac{1}{\kappa}}$ is of the form
\[
v^* = (\kappa, \cdots, \kappa, r v_{|\mathcal{J}^*|+1}, \cdots, r v_n),
\]
with $J^*\triangleq\{i\in[n]\text{ s.t. }v_i^*=\kappa\}$ and $r= (1 - |J^*|\kappa)/\sum_{j=|J^*|+1}^{n}{v_j}$.
Hence, Algorithm \ref{algo:proj_A} terminates the while loop with an index $i\leq |J^*|$. 
By contradiction, assume that Algorithm \ref{algo:proj_A} terminates at $i_0< |J^*|$, then the algorithm guarantees $\tilde{v}=(\kappa,\cdots,\kappa, \frac{1 - i_0 \kappa}{\sum_{j=i_0+1}^{n}{v_j}} v_{i_0+1}, \cdot , \frac{1 - i_0 \kappa}{\sum_{j=i_0+1}^{n}{v_j}} v_n) \in \mathcal{S}_{n,\frac{1}{\kappa}}$. 
However, with Lemma \ref{lem:projection_S_n_kappa} \ref{itm:third_claim}, this means that $D_F(\tilde{v},v)\leq D_F(v^*, v)$. Hence, $\tilde{v}\neq v^*$ is a feasible optimal solution to the Bregman projection problem of $v$ on $\mathcal{S}_{n,\frac{1}{\kappa}}$, contradicting the unicity of $v^*$.\\\\
Finally, let us show that $p_t$ as defined in Line \ref{line:def_p_t_i} of Algorithm \ref{algo:online_bandit_k_set} belongs to $\mathcal{S}_{n,k}$.
With $\alpha_t>0$, we have $p_{t,i}\geq 0$.
Besides, by definition of $\alpha_t$ in Line \ref{line:def_alpha} of Algorithm \ref{algo:online_bandit_k_set}, we have
\begin{eqnarray*}
    \sum_{j=1}^{n}{p_{t,j}} &=& (1-\gamma) \sum_{j=1}^{n}{v^{\prime}_{t,j}} + \gamma = (1-\gamma) \frac{\kappa}{\alpha_t}\Big(\sum_{j\notin J_t}{\alpha_t} + \sum_{j\in J_t}{v_{t,j}} \Big) + \gamma\\
    &=& (1-\gamma) \frac{\kappa}{\alpha_t} \frac{\alpha_t}{\kappa} + \gamma = 1.
\end{eqnarray*}
Then, for $i\in J_t$, $v^\prime_{t,i}= \frac{\kappa}{\alpha_t} \alpha_t = \kappa$. Hence,
\[
p_{t,i} = (1-\gamma) \kappa + \gamma/n = 1/k - \gamma /n + \gamma/n = 1/k.
\]
Finally, for $i\notin J_t$, $v^\prime_{t,i}\leq \kappa$ and we conclude that $p_{t,i}\leq 1/k$.
\end{proof}
\noindent Algorithm~\ref{algo:proj_A} sorts a vector of size $n$ in Line~\ref{line:sort} resulting in time complexity $\mathcal{O}(n \log n)$.
Note that \citep{herbster2001tracking} describe an algorithm \citep[Figure 3]{herbster2001tracking} that achieves an $\mathcal{O}(n)$ complexity by avoiding the sorting step.
However, for simplicity we do not consider this version since the extra logarithmic cost is not crucial.

\begin{lemma}
The time and space complexity of Algorithm~\ref{algo:proj_A} is $\mathcal{O}(n \log n)$ and $\mathcal{O}(n)$, respectively.
\end{lemma}

\begin{proof}
The projection in Line \ref{line:p} costs $\mathcal{O}(n)$.
The sorting in Line \ref{line:sort} costs $\mathcal{O}(n \log n)$.
By Lemma~\ref{lem:termination_projection}, the while-loop in Algorithm~\ref{algo:proj_A} gets executed at most $k$ times and Lines \ref{line:while_max}-\ref{line:increment} have time complexity $\mathcal{O}(k^2 + n)$.
In summary, the time complexity of Algorithm~\ref{algo:proj_A} is $\mathcal{O}(n \log n)$.
Since we only have to store $p, o, U, v$, and $J$, the space complexity of Algorithm~\ref{algo:proj_A} is $\mathcal{O}(n)$.
\end{proof}

\subsection{DepRound}\label{app:dep_round}
The sampling Algorithm DepRound (Algorithm \ref{algo:depRound}) used in Algorithm \ref{algo:online_bandit_k_set} samples a set of $k$ distinct actions from $[n]$ while satisfying the condition that each action $i$ is selected with probability $p_i$ \citep{vural2019minimax,uchiya2010}.
More formally it addresses the following problem:
\begin{problem}\label{problem:big_distr}
Given a vector $p = (p_1, \ldots, p_n)^T\in \cS_{n, k}$, sample a subset $S \in \cP_k([n])$ of size $k$ such that
\begin{equation}\label{eq:probabilities}
    \mathbb{P}[i \in S] = k p_i \qquad \text{ for all } i = 1,\ldots, n.
    \end{equation}
\end{problem}
\noindent In Line \ref{line:probas} of Algorithm \ref{algo:depRound}, at least one of $p_i, p_j$ becomes either $0$ or $1$. 
Thus, the inside of the while-loop is executed at most $n$ times.
Since the inside of the while-loop runs in time $\mathcal{O}(1)$, the time complexity of DepRound is $\mathcal{O}(n)$.
Further, the space complexity of Algorithm~\ref{algo:depRound} is $\mathcal{O}(n)$ as we only have to store $p, \alpha$, and  $\beta$.
\begin{algorithm}[ht]
  \caption{DepRound.}
  \label{algo:depRound}
  \begin{algorithmic}[1]
  \State \textbf{Input:} $k\in  [n]$, $p \in\cS_{n, k}$
  \State $p \gets k \cdot p$  \label{line:stretch}
  \While{$\exists \ i$ such that $p_i \in (0, 1)$}
    \State Select distinct $i, j \in [n]$ such that $p_i, p_j \in (0,1)$ \label{line:distinct}
    \State $\alpha \gets \min\{1-p_i, p_j\}$
    \State $\beta \gets \min\{p_i, 1-p_j\}$
    \State $(p_i, p_j) \gets \begin{cases} (p_i + \alpha, p_j - \alpha) \text{ with probability } \frac{\beta}{\alpha + \beta} \\
    (p_i - \beta, p_j + \beta) \text{ with probability } \frac{\alpha}{\alpha + \beta} \end{cases}$\label{line:probas} $\hfill \vartriangleright \text{Sample}$
  \EndWhile
  \State \textbf{Return: } $\{i \in [n] \mid p_i = 1\}$
  \end{algorithmic}
\end{algorithm}

\section{Omitted Proofs}\label{app:omitted_proofs}
\subsection{Online-bandit proof.}
Let us now prove the convergence of Algorithm \ref{algo:online_bandit_general} towards an approximate Nash equilibrium. The proofs of convergence for Algorithms \ref{algo:online_bandit_simplex} and \ref{algo:online_bandit_k_set} follow from this result.
The proof is classical in the online-online setting (see \textit{e.g.} \citep{wang2018acceleration} and references therein).
\begin{theorem}\label{thm:general}
Consider running $T$ rounds of Algorithm \ref{algo:online_bandit_general} with a no-regret online learning algorithm for the $w$-player ensuring a worst case regret $R_T^w\leq \epsilon_w(T)$, and a no-regret bandit algorithm for the $p$-player ensuring an expected regret of $\mathbb{E}[R_T^p]\leq \epsilon_p(T)$ or a regret of $R^p_T\leq \epsilon_p(T,\delta)$ that holds with probability $1-\delta$.
Write $a_t$ the actions of the $p$-player and $w_t$ the actions of the $w$-player. 
Then we have either

\begin{align}
    &\mathbb{E}[\Delta(\Bar{w},\Bar{a})] \leq \frac{1}{T}(R_T^w + \mathbb{E}[R_T^p])\leq \frac{1}{T}\Big(\epsilon_w(T) + \epsilon_p(T)\Big) ,\label{eq:exp_gap}\\
    \intertext{or}
    &\Delta(\Bar{w},\Bar{a}) \leq \frac{1}{T}(R_T^w + R_T^p)\leq \frac{1}{T}\Big(\epsilon_w(T) + \epsilon_p(T,\delta)\Big),\quad \text{w.p. $1-\delta$,}\label{eq:p_gap}
\end{align}
where $\Bar{a}=\frac{1}{T}\sum_{t=1}^{T}a_t$ and $\Bar{w}=\frac{1}{T}\sum_{t=1}^{T}w_t$. 
\end{theorem}

\begin{proof}
First, note that if one uses an OL algorithm with a regret guarantee that holds in expectation, the dual gap also holds in expectation, \textit{i.e.},
\begin{equation}
    \mathbb{E}\left[\Delta(w^\prime, p^\prime)\right]:=\mathbb{E}\left[ \underset{p\in\mathcal{K}}{\max}~\langle L(w^\prime); p\rangle - \underset{w\in\mathcal{W}}{\min}~\langle L(w); p^\prime\rangle\right].\label{eq:exp_dual_gap}\tag{Expected Dual Gap}
\end{equation}
The high-level idea of this proof is to show that combining the regret of both players will allow us to find an upper bound on the (expected) dual gap.
Let us write the regret guarantee for the $w$-player as well as the expected and high-probability regret guarantees for the $p$-player,
\begin{align*}
    R_T^w &:= \sum_{t=1}^{T} \langle L(w_t); a_t\rangle - \underset{w\in\mathcal{W}}{\text{min }} \sum_{t=1}^{T} \langle L(w); a_t\rangle\leq \epsilon_w(T),\\
    \mathbb{E}[R_T^p] &:=\mathbb{E}\left[ \underset{a\in\text{Ext}(\mathcal{K})}{\text{max }}\sum_{t=1}^{T} \langle L(w_t); a\rangle -  \sum_{t=1}^{T} \langle L(w_t); a_t\rangle\right]\leq \epsilon_p(T),\\
    R_T^p &:= \underset{a\in\text{Ext}(\mathcal{K})}{\text{max }}\sum_{t=1}^{T} \langle L(w_t); a\rangle -  \sum_{t=1}^{T} \langle L(w_t); a_t\rangle\leq \epsilon_p(T,\delta).
\end{align*}
By linearity of $a\mapsto \langle L(w_t); \cdot\rangle$ and convexity of $w\mapsto\langle L(\cdot); a_t\rangle$, we have
\begin{align}
    \frac{1}{T}\sum_{t=1}^{T} \langle L(w_t); a_t\rangle - \underset{w\in\mathcal{W}}{\text{min }}\langle L(w); \bar{a}\rangle &=\frac{1}{T}R_T^w \leq \frac{1}{T}\epsilon_w(T),\label{eq:reg_w}\\
     \mathbb{E}\left[ \underset{a\in\text{Ext}(\mathcal{K})}{\text{max }}\langle L(\bar{w}); a\rangle -  \frac{1}{T}\sum_{t=1}^{T} \langle L(w_t); a_t\rangle\right] &\leq \frac{1}{T}\mathbb{E}[R_T^p] \leq \frac{1}{T}\epsilon_p(T),\label{eq:reg_p}\\
    \underset{a\in\text{Ext}(\mathcal{K})}{\text{max }}\langle L(\bar{w}); a\rangle -  \frac{1}{T}\sum_{t=1}^{T} \langle L(w_t); a_t\rangle &\leq \frac{1}{T}R_T^p \leq \frac{1}{T}\epsilon_p(T,\delta) , \quad \text{w.p. }1-\delta.\label{eq:high_p_reg_p}
\end{align}
Where $\bar{w}= \frac{1}{T}\sum_{t=1}^Tw_t$ and $\bar{a}= \frac{1}{T}\sum_{t=1}^Ta_t$. Adding \eqref{eq:reg_w} and \eqref{eq:reg_p}, we obtain
\begin{align*}
     \frac{1}{T}\epsilon_w(T) + \frac{1}{T}\epsilon_p(T) &\geq \mathbb{E}\left[ \underset{a\in\text{Ext}(\mathcal{K})}{\text{max }}\langle L(\bar{w}); a\rangle\right] - \underset{w\in\mathcal{W}}{\text{min }}\langle 
     L(w); \bar{a}\rangle\\
     &=\mathbb{E}\left[ \underset{a\in\text{Ext}(\mathcal{K})}{\text{max }}\langle L(\bar{w}); a\rangle - \underset{w\in\mathcal{W}}{\text{min }}\langle L(w); \bar{a}\rangle\right]\\
     &= \mathbb{E}\left[\Delta(\bar{w},\bar{a})\right].
\end{align*}
Further, adding \eqref{eq:reg_w} and \eqref{eq:high_p_reg_p}, we obtain 
\begin{align*}
     \frac{1}{T}\epsilon_w(T) + \frac{1}{T}\epsilon_p(T,\delta) &\geq \underset{a\in\text{Ext}(\mathcal{K})}{\text{max }}\langle L(\bar{w}); a\rangle - \underset{w\in\mathcal{W}}{\text{min }}\langle L(w); \bar{a}\rangle\\
     &= \underset{a\in\text{Ext}(\mathcal{K})}{\text{max }}\langle L(\bar{w}); a\rangle - \underset{w\in\mathcal{W}}{\text{min }}\langle L(w); \bar{a}\rangle\\
     &= \Delta(\bar{w},\bar{a}).
\end{align*}
We have retrieved \eqref{eq:exp_dual_gap} and \eqref{eq:dual_gap} respectively on the right side of both inequalities, proving \eqref{eq:exp_gap} and \eqref{eq:p_gap}. 
Indeed, the sum of the average regrets of both players constitute an upper bound to the (expected) dual gap.
\end{proof}

\subsection{Proofs with the $n$-Simplex}\label{ssec:proof_simplex}
Theorem \ref{th:shalev} is similar to that provided in \citep[Theorem 1.]{shalev2016minimizing} for the convex-linear case.
Interestingly \citep[Theorem 1.]{shalev2016minimizing} does not in general require convexity.
In the non-convex case, the theorem gives an error bound for an ensemble of predictors parametrized by randomly selected of iterates $w_t$.
For the convex-linear case, the error bound holds for just one predictor parametrized by the average iterate $\bar{w}$.
However, the result relies on the separability assumption (\textit{i.e.},there exists $\tilde{w}\in\mathcal{W}$ s.t. $\ell(\tilde{w}, x_i, y_i)=0$ for all $i\in [n]$).
This is problematic as this does not typically hold for real-world datasets.
Further Algorithm \ref{algo:online_bandit_simplex} uses EXP.IX instead of EXP.3P which leads to slightly faster convergence.
Lastly, Theorem \ref{th:shalev} differs from \citep[Theorem 1.]{shalev2016minimizing} in that it gives a guarantee on the convergence of the game instead of a guarantee on the quality of the prediction.
\begin{theorem}
Let $\delta>0$. Consider running $T$ rounds of Algorithm \ref{algo:online_bandit_simplex} with a choice of online learning algorithm for the $w$-player ensuring a worst case regret $R_T^w\leq C \sqrt{T}$ for some $C>0$.
Further fix the parameters
\[
\eta_t=2\gamma_t=\sqrt{\frac{2\log n}{nT}}.
\]
Write $a_t$ the actions of the bandit and $w_t$ that of the online $w$-player.
Then with probability $1-\delta$, we have
\begin{align*}
    \Delta(\Bar{w},\Bar{a}) \leq C\sqrt{\frac{1}{T}} + 2\sqrt{\frac{2n\log(n)}{T}} + \left(\sqrt{\frac{2n}{T\log{n}}}+\frac{1}{T}\right)\log(2\delta^{-1}),
\end{align*}
where $\Bar{a}=\frac{1}{T}\sum_{t=1}^{T}a_t$ and $\Bar{w}=\frac{1}{T}\sum_{t=1}^{T}w_t$.
\end{theorem}

\begin{proof}
The proof follows from Theorem \ref{thm:general} when the $w$-player chooses Online Gradient Descent (OGD) and the $p$-player chooses EXP.IX. 
When $\ell(h(\cdot,x_{i}),y_{i})$ is $L$-Lipschitz, and the parameters $w$ lie in a $\ell_2$-Ball of size $B$, running OGD with step size of $\eta=\frac{B}{L\sqrt{2T}}$, we have \citep[Corollary 2.7]{shalev2011online},
\[
\frac{1}{T}\epsilon_w(T)\leq BL\sqrt{2/T}.
\]
Similarly, from \citep[Theorem 1]{neu2015explore} with the choice of $\eta_{t}=2\gamma_{t}=\sqrt{\frac{2\log n}{nT}}$ in EXP.IX we have with probability at least $1-\delta$, 
\[
\frac{1}{T}\epsilon_p(T,\delta)\leq 2\sqrt{2n\log( n)/T}+\left(\sqrt{\frac{2n}{T\log{n}}}+\frac{1}{T}\right)\log(2\delta^{-1}),
\]
which concludes the proof.
\end{proof}

\subsection{Proofs with $k$-set}\label{ssec:proof_k_set}
We start with an auxiliary theorem that provides a high-probability upper-bound on an stationary random process.
Let $X_{1},\ldots,X_{T}$ be a sequence of real-valued random variables. Let $\mathbb{E}_{t}[Y]=\mathbb{E}[Y\vert X_{1},\ldots,X_{t-1}]$.

\begin{theorem}[Theorem $1$ \citep{beygelzimer2011contextual}]\label{thm:e2p}
Let $R>0$ and assume that $X_{t}\leq R$ and $\mathbb{E}_{t}[X_{t}]=0$ for all $t$. Define the random variables
\begin{align*}
    S =\sum^{T}_{t=1}X_{t} \qquad \text{ and } \qquad \sigma=\sum_{t=1}^{T}\mathbb{E}_{t}[X_{t}^{2}].
\end{align*}
Then, for any $\delta>0$, with probability at least $1-\delta$,
\begin{equation*}
  S \leq \left\{
    \begin{split}
    &\sqrt{(e-2)\log (1/\delta)}\left(\frac{\sigma}{\sqrt{\sigma'}}+\sqrt{\sigma'} \right), &\text{ for any }& \sigma'\in \left[ \frac{R^{2}\log (1/\delta)}{e-2},\infty \right)\\
    &R\log(1/\delta) +(e-2)\frac{\sigma}{R}, & \text{otherwise}
    \end{split}
  \right. .
\end{equation*}
\end{theorem}
With this result, we can now provide a proof of a high-probability regret bound as in \citep{vural2019minimax} for the EXP4.MP used in Algorithm \ref{algo:online_bandit_k_set}.
    \begin{theorem}\label{th:EXP4.MP}
    For $T\geq \max \{\log(n/\delta), n\log(n/k)/k \}$ and
    \begin{align}\label{eq:def_hyper_parameters}
        \eta=\frac{k\gamma}{2n},\quad \gamma=\sqrt{\frac{n\log (n/k)}{kT}},\quad c=\sqrt{k\log(n/\delta)},
    \end{align}
    EXP4.MP guarantees
    \begin{align*}
        R_T\leq 2\sqrt{knT\log(n/\delta)}+4\sqrt{knT\log(n/k)} + k\log(n/\delta)
    \end{align*}
    with probability $1-\delta$.
    \end{theorem}
    
    \begin{proof}
    The proof is a simplification of \citep{vural2019minimax} to our setting that we repeat for completeness.
    If not stated otherwise, all Line references refer to Algorithm \ref{algo:online_bandit_k_set}.
    Note that while we write $L$ for the feedback, this proof is formulated in terms of gains.
    This comes from the fact that in Algorithm \ref{algo:online_bandit_k_set}, the bandit is trying to maximize $L$.
    To make the proof easier to grasp, we have split it into five parts.
    In the second and third part, we derive an upper bound to $\log \frac{P_{T+1}}{P_{1}}$ in \eqref{eq:upper} and a lower bound in \eqref{eq:log_lower}.
    Then we combine \eqref{eq:upper} and \eqref{eq:log_lower} to obtain \eqref{eq:simple_ineq}.
    In the last part, we use the concentration bounds from Theorem \ref{thm:e2p} on \eqref{eq:simple_ineq} to give a high probability bound on the regret.
    
    \paragraph{Notation.}
    Let us introduce some notations,
    \begin{align*}
      P_{t}&=\sum^{n}_{i=1}\tilde{p}_{t,i}, \quad  \tilde{L}_{t,i}=\hat{L}_{t,i}+\frac{c}{\sqrt{nT}} U_{t,i}, \quad v_{t,i}=\frac{\tilde{p}_{t,i}}{P_t}\\
      G_{i}&= \sum_{t=1}^{T}L_{t,i},\quad \hat{G}_{i}=\sum_{t=1}^{T}\hat{L}_{t,i}, \quad 
      \hat{\Gamma}_{T,I}= \sum_{i\in I}\hat{G}_{i}+ \frac{c}{\sqrt{nT}}\sum_{t=1}^{T}\sum_{i\in I}U_{t,i}.
    \end{align*}
    Here, $v_t=(v_{t,1},\cdots,v_{t,n})^T$ is a probability distribution obtained by normalizing the weights $\tilde{p}_t$ (Line \ref{line:proj}).
    The vector $v^\prime_t$ will be a \emph{capped} version of $v_t$ (Line \ref{line:def_J_t}), hence not necessarily a probability vector.
    Further, the vector $p_t$ in Line \ref{line:def_p_t_i} is the projection of $\tilde{p}_t$ onto $\mathcal{S}_{n,k}$.
    Formally, the quantity $\hat{\Gamma}_{T,I}$ is the scalar product between an action of the bandit (associated to the subset $I$ of $[n]$) and a cumulative cost vector over the $T$ rounds of the game.
    We now define $I^*$ as the action maximizing this scalar product, \textit{i.e.}, the best action in hindsight
    \begin{equation}\label{eq:best_action_proof}
        I^{*}=\argmax_{I\in \mathcal{P}_k([n])}\hat{\Gamma}_{T,I}  = \sum_{t=1}^{T}\sum_{i\in I}{\tilde{L}_{t,i}}.
    \end{equation}
    \paragraph{Upper bound.}
    By the definition of $P_t$ and the update rules for $\tilde{p}$ (\textit{i.e.}, Lines \ref{algo:estL}-\ref{algo:expupdate}), we have
    \begin{align*}
      \frac{P_{t+1}}{P_{t}}=\frac{1}{P_{t}}\left(\sum_{i\in J_{t}}\tilde{p}_{t+1,i}+\sum_{i\notin J_{t}}\tilde{p}_{t+1,i}\right)=\frac{1}{P_{t}}\left(\sum_{i\in J_{t}}\tilde{p}_{t,i}+\sum_{i\notin J_{t}}\tilde{p}_{t,i}\exp (\eta \tilde{L}_{t,i})\right).
    \end{align*}
    From \eqref{eq:def_hyper_parameters}, we have $\eta = (k\gamma)/(2n)$ and $c=\sqrt{k\log(n/\delta)}$. 
    Further, for $i\notin J_t$ we have $\tilde{L}_{t,i}=1_{[i\in I_t]}L_{t,i}/(kp_{t,i}) + c/(kp_{t,i}\sqrt{nT})$.
    Since the loss function takes its values in $[0,1]$,
    from Line \ref{line:loss_extraction}, we have $L_{t,i}\leq 1$.
    From Line \ref{line:def_p_t_i}, we know that the probability $p_t$ is a convex combination with the uniform distribution, ensuring that $p_{t,i}\geq \gamma/n$.
    Finally, with $k\leq n$, we have
    \begin{equation*}
        \eta \tilde{L}_{t,i} = \frac{k\gamma}{2n} \left[\frac{L_{t,i}}{k p_{t,i}}1_{[i\in I_t]} + \frac{c}{k p_{t,i}\sqrt{nT}} \right] \leq\frac{k\gamma}{2n} \left[\frac{n}{k\gamma} + \frac{cn}{\gamma k \sqrt{nT}}\right] = \frac{1}{2} \left[1 + \frac{\sqrt{k\log(n/\delta)}}{\sqrt{nT}}\right] \leq \frac{1}{2}\left[1 + \frac{\sqrt{\log(n/\delta)}}{\sqrt{T}}\right].
    \end{equation*}
    Hence, provided that $T\geq \log(n/\delta)$, we have $\eta \tilde{L}_{t,i}\leq 1$.
    Then, note that $e^{a}\leq 1+a+a^{2}$ for $a\leq 1$.
    We now have
    \begin{align*}
      \frac{P_{t+1}}{P_{t}}&\leq \frac{1}{P_{t}}\left(\sum_{i\in J_{t}}\tilde{p}_{t,i}+\sum_{i\notin J_{t}}\tilde{p}_{t,i}\left(1+\eta\tilde{L}_{t,i}+\eta^{2}\tilde{L}_{t,i}^{2}\right) \right)\\
      &= \frac{1}{P_{t}}\left(\sum_{i=1}^{n}\tilde{p}_{t,i}+\sum_{i\notin J_{t}}\tilde{p}_{t,i}\left(\eta\tilde{L}_{t,i}+\eta^{2}\tilde{L}_{t,i}^{2}\right) \right)\\
      &= 1 + \frac{1}{P_{t}}\sum_{i\notin J_{t}}\tilde{p}_{t,i}(\eta\tilde{L}_{t,i}+\eta^{2}\tilde{L}_{t,i}^{2})= 1 + \sum_{i\notin J_{t}}v_{t,i}\left(\eta\tilde{L}_{t,i}+\eta^{2}\tilde{L}_{t,i}^{2}\right).
    \end{align*}
    Using $1+x\leq e^{x}$, we get
    \begin{align*}
      \frac{P_{t+1}}{P_{t}}\leq 1 + \sum_{i\notin J_{t}}v_{t,i}\left(\eta\tilde{L}_{t,i}+\eta^{2}\tilde{L}_{t,i}^{2}\right)\leq \exp\left( \sum_{i\notin J_{t}}v_{t,i}\left(\eta\tilde{L}_{t,i}+\eta^{2}\tilde{L}_{t,i}^{2}\right) \right).
    \end{align*}
    Then taking the logarithm and summing over $t=1,\ldots,T$,
    \begin{align}\label{eq:log_upper}
        \log \left(\frac{P_{T+1}}{P_{1}}\right) \leq  \sum^{T}_{t=1}\sum_{i\notin J_{t}}v_{t,i}\left(\eta\tilde{L}_{t,i}+\eta^{2}\tilde{L}_{t,i}^{2}\right).
    \end{align}
    Let us now prove the following auxiliary result needed to further upper bound the terms in \eqref{eq:log_upper},
    \begin{align}\label{eq:p_upper}
        \frac{\tilde{p}_{t,i}}{P_t} = v_{t,i}\leq \frac{v'_{t,i}}{\sum_{j=1}^{n}v'_{t,j}}\leq \frac{p_{t,j}}{1-\gamma},\quad \forall i \notin J_{t},
    \end{align}
    where $J_t=\{i\in[n]~|~v_{t,i}\geq \alpha \}$ (Line \ref{line:def_J_t}). Further $v'_{t,i}=v_{t,i}$ for $i\notin J_t$ and $v'_{t,i} = \alpha$ otherwise, where $\alpha$ is defined as in Line \ref{line:def_alpha}.
    Indeed, for $i\notin J_t$, we have $v'_{t,i}=v_{t,i}$. 
    Hence, since $\sum_{i=1}^nv'_{t,i}\leq \sum_{i=1}^nv_{t,i}=1$ and with the update rule of $p_T$, we obtain
    \[
    v_{t,i}\leq \frac{v'_{t,i}}{\sum_{j=1}^{n}v'_{t,j}}= \frac{p_{t,i} - \gamma/n}{1-\gamma} \leq \frac{p_{t,i}}{1-\gamma}.
    \]
    Recall that from Line \ref{algo:estL}, we have $\hat{L}_{t,i}=1_{[i\in I_t\backslash J_t]}L_{t,i}/(k p_{t,i})$. Using \eqref{eq:p_upper} and since $L_{t,i}\leq 1$, we obtain
    \begin{align}\label{eq:lhat}
        \sum_{i\notin J_t}v_{t,i}\hat{L}_{t,i}^{2}\leq \frac{1}{1-\gamma} \sum_{i\notin J_t} p_{t,i}\hat{L}_{t,i}^{2}\leq \frac{1}{k(1-\gamma)} \sum_{i\in I_t\backslash J_t} \hat{L}_{t,i}.
    \end{align}
    Now let us upper bound the terms in \eqref{eq:log_upper}. 
    Using \eqref{eq:p_upper} we upper bound the first term,
    \begin{align}
        \sum_{i\notin J_t}v_{t,i}\tilde{L}_{t,i}&= \sum_{i\notin J_t}v_{t,i}\left( \hat{L}_{t,i}+ \frac{c}{kp_{t,i}\sqrt{nT}}\right)\nonumber\\
        &\leq \frac{1}{1-\gamma}\sum_{i\notin J_{t}} \left[p_{t,i}\hat{L}_{t,i}+\frac{c}{k\sqrt{nT}}\right]\nonumber\\
        &\leq \frac{1}{1-\gamma}\sum_{i\notin J_{t}} \left[p_{t,i}\hat{L}_{t,i}\right]+\frac{c(n-|J_t|)}{(1-\gamma)k\sqrt{nT}}\nonumber\\
        &\leq \frac{1}{1-\gamma}\sum_{i\notin J_{t}} \left[p_{t,i}\hat{L}_{t,i}\right]+\frac{c}{(1-\gamma)k}\sqrt{\frac{n}{T}}.\label{eq:l}
    \end{align}
    We go on to bound the second term in \eqref{eq:log_upper}. 
    For the first inequality, we use $(a+b)^{2}\leq 2(a^{2}+b^{2})$. 
    Then, we use \eqref{eq:lhat} for the second inequality and $U_{t,i}=\frac{1}{k p_{t,i}}\leq \frac{n}{\gamma k}$ for the third.
    For the last line, we use $\hat{L}_{t,i}=1_{[i\in I_t\backslash J_t]}L_{t,i}/(k p_{t,i})$ with $L_{t,i}\leq 1$,
    \begin{align}
        \sum_{i\notin J_t}v_{t,i}\tilde{L}_{t,i}^{2}&= \sum_{i\notin J_t}v_{t,i}\left( \hat{L}_{t,i}+ \frac{c}{\sqrt{nT}}U_{t,i}\right)^{2}\nonumber\\
        &\leq  2\sum_{i\notin J_t}v_{t,i}\left( \hat{L}_{t,i}^{2}+ \left(\frac{cU_{t,i}}{\sqrt{nT}}\right)^{2}\right)\nonumber\\
        &\leq \frac{2}{k(1-\gamma)} \sum_{i\in I_t\backslash J_t}\hat{L}_{t,i} + \frac{2c^2}{nT} \sum_{i\notin J_t} \frac{p_{t,i}}{1-\gamma} \frac{1}{k p_{t,i}} U_{t,i}\nonumber\\
        &\leq \frac{2}{k(1-\gamma)} \sum_{i\in I_t\backslash J_t}\hat{L}_{t,i} + \frac{2c^2(n-|J_t|)}{knT(1-\gamma)} \frac{n}{\gamma k}\nonumber\\
        &\leq  \frac{2}{k(1-\gamma)} \sum_{i\in I_t\backslash J_t}\hat{L}_{t,i} + \frac{2c^{2}n}{k^{2}T(1-\gamma)\gamma}. \label{eq:l2}
    \end{align}
    We now plug in previously derived bounds from \eqref{eq:l} and \eqref{eq:l2} into \eqref{eq:log_upper} to obtain
    \begin{align*}
        \log \left(\frac{P_{T+1}}{P_{1}}\right) \leq & \sum^{T}_{t=1}\sum_{i\notin J_{t}}v_{t,i}\left(\eta\tilde{L}_{t,i}+\eta^{2}\tilde{L}_{t,i}^{2}\right) \\
        \leq & \eta \sum^{T}_{t=1} \left(\frac{1}{1-\gamma}\sum_{i\notin J_{t}}p_{t,i}\hat{L}_{t,i}+\frac{c}{k(1-\gamma)}\sqrt{\frac{n}{T}}\right) \\
        &+ \eta^{2}\sum_{t=1}^{T}\left(\frac{2}{k(1-\gamma)} \sum_{i\in I_t\backslash J_t}\hat{L}_{t,i} + \frac{2c^{2}n}{k^{2}T(1-\gamma)\gamma}\right)\\
        \leq &\frac{\eta}{1-\gamma} \sum^{T}_{t=1}\sum_{i\notin J_t}p_{t,i}\hat{L}_{t,i}+\frac{\eta c\sqrt{nT}}{k(1-\gamma)}\\
        &+ \frac{2\eta^{2}}{k(1-\gamma)} \sum_{t=1}^{T}\sum_{i\in I_t\backslash J_t}\hat{L}_{t,i} + \frac{2\eta^{2}c^{2}n}{ k^{2}(1-\gamma)\gamma}.\label{eq:upper}\tag{Upper-Bound}
    \end{align*}

    \paragraph{Lower bound.}
    Conversely, we can lower bound $\log \big(P_{T+1}/P_{1}\big)$ using $I^*$ (of size $k$) as defined in \eqref{eq:best_action_proof} as follows
    \begin{equation*}
      \log\left(\frac{P_{T+1}}{P_{1}} \right) = \log\left(\frac{\sum_{i=1}^n\tilde{p}_{T+1,j}}{P_{1}}\right) \geq \log\left( \frac{\sum_{j\in I^{*}}\tilde{p}_{T+1,j}}{P_{1}}\right)= \log\left(\sum_{j\in I ^{*}}\tilde{p}_{T+1,j}\right)-\log(P_{1}).
    \end{equation*}
    Using the inequality of arithmetic and geometric means, we obtain 
    \begin{align*}
        \log\left(\frac{P_{T+1}}{P_{1}} \right)&\geq \log\left(\sum_{j\in I ^{*}}\tilde{p}_{T+1,j}\right)-\log\left(P_{1}\right) \\
        &\geq \log\left(k \Big(\prod_{j\in I^{*}}\tilde{p}_{T+1,j}\Big)^{\frac{1}{k}}\right) - \log\left(P_{1}\right) \\
        &\geq \frac{1}{k} \sum_{j\in I^{*}} \log (\tilde{p}_{T+1,j})-\log\left( \frac{P_{1}}{k}\right).
    \end{align*}
    We now consider the update rule of the $\tilde{p}_{t}$ in Line \ref{algo:expupdate} of Algorithm \ref{algo:online_bandit_k_set},
    \begin{align*}
      \forall i\in[n],~~\tilde{p}_{T+1,i}= \tilde{p}_{1,i}\exp\left(\eta \sum ^{T}_{t=1}\tilde{L}_{t,i}\right).
    \end{align*}
    Taking the logarithm and summing over all actions $j\in I^*$, we have
    \begin{align*}
        \sum_{j\in I^{*}} \log \tilde{p}_{T+1,j}= \sum_{j\in I^{*}}\left[\log \tilde{p}_{1,j}+ \eta \sum^{T}_{t=1}\tilde{L}_{t,j}\right].
    \end{align*}
    This equality allows us to write the following lower bound, 
    \begin{align}\label{eq:log_lower}\tag{Lower-Bound}
       \log\left(\frac{P_{T+1}}{P_{1}}\right) \geq \frac{1}{k}\sum_{j\in I^{*}}\left[\log \tilde{p}_{1,j}+ \eta \sum ^{T}_{t=1}\tilde{L}_{t,j}\right] -\log\left(\frac{P_{1}}{k}\right).
    \end{align}
    \paragraph{Central inequality.}
    Then, using \eqref{eq:upper} and \eqref{eq:log_lower}, we get
    \begin{equation}\label{eq:upper_lower_bound}
        \sum_{j\in I^{*}}\left[\log \tilde{p}_{1,j}+ \frac{\eta}{k} \sum ^{T}_{t=1}\tilde{L}_{t,j}\right] -\log \frac{P_{1}}{k}\leq \log \frac{P_{T+1}}{P_{1}} \leq  \sum^{T}_{t=1}\sum_{i\notin J_{t}}v_{t,i}\left(\eta\tilde{L}_{t,i}+\eta^{2}\tilde{L}_{t,i}^{2}\right). 
    \end{equation}
    We initialize the weights uniformly $\tilde{p}_{1,i}=1$.
    Hence, \eqref{eq:upper_lower_bound} becomes
    \[
    \frac{\eta}{k} \sum_{j\in I^*}\sum ^{T}_{t=1}\tilde{L}_{t,j}-\log \frac{n}{k} \leq \frac{\eta}{1-\gamma} \sum^{T}_{t=1}\sum_{i\notin J_t}p_{t,i}\hat{L}_{t,i}+\frac{\eta c\sqrt{nT}}{k(1-\gamma)} + \frac{2\eta^{2}}{k(1-\gamma)} \sum_{t=1}^{T}\sum_{i\in I_t\backslash J_t}\hat{L}_{t,i} + \frac{2\eta^{2}c^{2}n}{\gamma k^{2}(1-\gamma)}.
    \]
    Then, we multiply by $\frac{k(1-\gamma)}{\eta}$ and use $\hat{L}_{t,i}=L_{t,i}/(k p_{t,i})1_{[i\in I_t\backslash J_t]}$ to obtain
    \begin{equation}\label{eq:to_be_referenced}
    (1-\gamma) \sum_{j\in I^*}\sum ^{T}_{t=1}\tilde{L}_{t,j} \leq \frac{k(1-\gamma)}{\eta}\log \frac{n}{k} +  \sum^{T}_{t=1}\sum_{i\in I_t\backslash J_t}L_{t,i} + c\sqrt{nT} + 2\eta \sum_{t=1}^{T}\sum_{i\in I_t\backslash J_t}\hat{L}_{t,i} + \frac{2\eta c^{2}n}{\gamma k}.
    \end{equation}
    From the optimality of $I^*$ in \eqref{eq:best_action_proof} it follows in particular that the average estimate $\tilde{L}_{t,i}$ over the indices of the optimal action is larger or equal to the average $\tilde{L}_{t,i}$ over all $n$ indices, \textit{i.e.},
    \begin{equation*}
        \frac{1}{n} \sum_{t=1}^T\sum_{i=1}^n\tilde{L}_{t,i}\leq \frac{1}{k} \sum_{t=1}^T\sum_{i\in I^*}\tilde{L}_{t,i}\quad \left(=\frac{1}{k}\hat{\Gamma}_{T,I^*}\right).
    \end{equation*}
    From $U_{t,i} \geq 0$ it follows that $\hat{L}_{t,i}\leq \tilde{L}_{t,i}$. 
    Using this fact and noting that since $U_{t,i} \geq 0$ we have that, $\hat{L}_{t,i}\leq \tilde{L}_{t,i}$, it follows,
    \[
    \sum_{t=1}^{T}\sum_{i=1}^{n} \hat{L}_{t,i} \leq \sum_{t=1}^{T}\sum_{i=1}^{n} \tilde{L}_{t,i} \leq \frac{n}{k} \hat{\Gamma}_{T, I^*}.
    \]
    Now, applying the above result in \eqref{eq:to_be_referenced} with $\eta=(k\gamma)/(2n)$ and the definition of $\hat{\Gamma}_{T, I^*}$, we obtain 
    \[
    (1-\gamma) \hat{\Gamma}_{T, I^*} \leq \frac{2n(1-\gamma)}{\gamma}\log \frac{n}{k} +  \sum^{T}_{t=1}\sum_{i\in I_t\backslash J_t}L_{t,i} + c\sqrt{nT} + \gamma \hat{\Gamma}_{T, I^*} + c^2.
    \]
    Further rearranging yields
    \[
    (1-3\gamma) \hat{\Gamma}_{T, I^*} \leq \frac{2n(1-\gamma)}{\gamma}\log \frac{n}{k} +  \sum^{T}_{t=1}\sum_{i\in I_t\backslash J_t}L_{t,i} + c\sqrt{nT} + c^2.
    \]
    Finally, with $\sum_{t=1}^{T}\sum_{i\in I_t\backslash J_t}L_{t,i}\leq \sum_{t=1}^{T}\sum_{i\in I_t}L_{t,i}= G$ we have, 
    \begin{equation}\label{eq:simple_ineq} 
        (1-2\gamma) \hat{\Gamma}_{T,I^*}-\frac{2n(1-\gamma)}{\gamma}\log \frac{n}{k}\leq G+c\sqrt{nT}+c^{2}.
    \end{equation}
    \paragraph{Concentration.}
    Let us now use the concentration result in Theorem \ref{thm:e2p} with the sequence $Y_{t,i}:= 1_{i\notin J_t}(L_{t,i}-\hat{L}_{t,i})$ for $i\in[n]$.
    In order to use Theorem \ref{thm:e2p}, we need the following terms
    \begin{align*}
        \mathbb{E}[Y_{t,i}]=0,\quad Y_{t,i}\leq 1,\quad \mathbb{E}[Y_{t,i}^{2}]\leq U_{t,i}.
    \end{align*}
    Note that by definition $\mathbb{P}\left(i\in I_t\right) = k p_{t,i}$. With this identity we can now prove these terms, 
    \begin{align*}
        \mathbb{E}[Y_{t,i}]&= \mathbb{E}\left[1_{i\notin J_t}\right]\left(L_{t,i}-\frac{L_{t,i}}{k p_{t,i}}\mathbb{P}\left(i\in I_t\right)\right) =  \mathbb{E}[1_{i\notin J_t}](L_{t,i}-L_{t,i}) = 0\\
        Y_{t,i} &= 1_{i\notin J_t}(\underbrace{L_{t,i}}_{\leq 1}-\underbrace{\frac{L_{t,i}}{k p_{t,i}}1_{i\in I_t}}_{\geq 0})\leq 1\\
        \mathbb{E}[Y_{t,i}^{2}] &= \mathbb{E}[1_{i\notin J_t}(L_{t,i}-\hat{L}_{t,i})^2] = \mathbb{E}[1_{i\notin J_t}\hat{L}_{t,i}^2] - \mathbb{E}[1_{i\notin J_t}\hat{L}_{t,i}]^2 \leq \mathbb{E}[1_{i\notin J_t}\hat{L}_{t,i}^2] \\
        &= \sum_{i=1}^n p_{t,i} 1_{i\in I_t}\frac{L_{t,i}^2}{k^2p_{t,i}^2}1_{i\notin J_t} = 1_{i\notin J_t} \frac{L_{t,i}^2}{kp_{t,i}}\leq 1_{i\notin J_t} \frac{1}{kp_{t,i}} \\
        &= U_{t,i}
    \end{align*}
    In particular, we have
    \[
    \sigma_i := \sum_{t=1}^{T}\mathbb{E}_t\big[Y_{t,i}^2\big] \leq \sum_{t=1}^{T} U_{t,i} = \sum_{t=1}^T{\frac{1}{k p_{t,i}}} \leq \frac{nT}{k\gamma}.
    \]
    Noting that $Y_{t,i}\leq 1=R$, we choose $T$ such that 
    \[\log(n/\delta)\leq (e-2)\frac{nT}{k}.\]
    Hence we can use the first case of Theorem \ref{thm:e2p} with $\sigma'=\frac{nT}{k}$. We obtain 
    \begin{align*}
        P\left(\sum^T_{t=1}Y_{t,i}\leq \sqrt{(e-2)\log\left(\frac{n}{\delta}\right)}\left(\frac{\sigma_i}{\sqrt{\sigma'}}+\sqrt{\sigma'}\right)\right)\geq 1-\frac{\delta}{n}.
    \end{align*}
    Note that
    \begin{align*}
        \frac{\sigma_i}{\sqrt{\sigma'}}+\sqrt{\sigma'}\leq \sqrt{\frac{k}{nT}}\sum^T_{t=1}U_{t,i} + \sqrt{\frac{nT}{k}}.
    \end{align*}
    Using this upper bound and the fact that $e-2<1$, we write
    \begin{align*}
        P\left(\sum^T_{t=1}Y_{t,i}\leq \sqrt{\log\left(\frac{n}{\delta}\right)}\left(\sqrt{\frac{k}{nT}}\sum^T_{t=1}U_{t,i} + \sqrt{\frac{nT}{k}}\right)\right)\geq 1-\frac{\delta}{n}.
    \end{align*}
    This concentration bound only holds for one fixed $i\in [n]$. In order to obtain a bound which holds for all $i\in[n]$ simultaneously, we take the union bound over all $i\in [n]$,
    \begin{align*}
        P\left(\forall i\in [n]:\sum_{t=1}^{T}Y_{t,i}\leq \sqrt{\log \left(\frac{n}{\delta}\right)}\left(\sqrt{\frac{k}{nT}}\sum^T_{t=1}U_{t,i} + \sqrt{\frac{nT}{k}}\right) \right)\geq 1-\delta.
    \end{align*}
    Since the bound now holds for all $i\in[n]$, we can sum over multiple $i$. In particular, we now sum over all $i$ in one action $I$,
    \begin{align*}
        P\left(\forall I\in \mathcal{P}_k\big([n]\big):\sum_{i\in I}\sum_{t=1}^{T}Y_{t,i}\leq \sqrt{\log \frac{n}{\delta}}\sum_{i\in I}\left(\sqrt{\frac{k}{nT}}\sum^T_{t=1}U_{t,i} + \sqrt{\frac{nT}{k}}\right)\right)\geq 1-\delta.
    \end{align*}
    Note that for an $i \in J_t$ for any $t\in [T]$, we have $p_{t,i}=1/k$ by Line \ref{line:def_p_t_i}. Hence,
    \begin{align}\label{eq:Lhat}
        \hat{L}_{t,i}  - L_{t,i}= \frac{L_{t,i}}{k\cdot 1/k}  - L_{t,i}= L_{t,i} - L_{t,i} = 0, \quad \forall i \in J_t.
    \end{align}
    Using \eqref{eq:Lhat} and by definition of $G_i$, $\hat{G}_i$ and $Y_{t,i}$, we can equivalently write
    \begin{align*}
        \sum_{i\in I} (G_i -\hat{G}_i)&= \sum_{i\in I} \left(\sum_t^T L_{t,i} - \sum_t^T \hat{L}_{t,i}\right)\\  &=\sum_{i\in I\backslash J_t} \left(\sum_t^T L_{t,i} - \sum_t^T \hat{L}_{t,i}\right) + \sum_{i\in I\cap J_t} \left(\sum_t^T L_{t,i} - \sum_t^T \hat{L}_{t,i}\right)\\
        &= \sum_{i\in I}\sum_{t=1}^{T} 1_{i\notin J_t}(L_{t,i}-\hat{L}_{t,i}) =\sum_{i\in I}\sum_{t=1}^{T}Y_{t,i}.
    \end{align*}
    Hence,
    \begin{align*}
        P\left(\sum_{i\in I}G_{i}-\hat{G}_{i}\leq \sqrt{\log \frac{n}{\delta}}\sum_{i\in I}\left(\sqrt{\frac{k}{nT}}\sum^T_{t=1}U_{t,i} + \sqrt{\frac{nT}{k}}\right)\right)\geq 1-\delta.
    \end{align*}
    We simplify the upper bound as follows
    \begin{align*}
         \sqrt{\log \frac{n}{\delta}}\sum_{i\in I}\left(\sqrt{\frac{k}{nT}}\sum^T_{t=1}U_{t,i} + \sqrt{\frac{nT}{k}}\right) = \sqrt{\log \frac{n}{\delta}}\left(\sqrt{\frac{k}{nT}}\sum_{t=1}^{T}\sum_{i\in I}U_{t,i}+\sqrt{knT}\right).
    \end{align*}
    With $c=\sqrt{k\log \frac{n}{\delta}}$ we obtain
    \begin{align*}
        \hat{\Gamma}_{T,I}= \sum_{i\in I} \hat{G}_{i} + \sqrt{\frac{k \log n/\delta }{nT}} \sum_{t=1}^{T}\sum_{i\in I} U_{t,i}.
    \end{align*}
    Using these two terms, we can rewrite the concentration bound as
    \begin{align*}
        P\left(\sum_{i\in I}G_{i}\leq \hat{\Gamma}_{T,I} + c\sqrt{nT}\right)\geq 1-\delta.
    \end{align*}
    Since the inequality holds for any $I\in \mathcal{P}_k\big([n]\big)$, it holds in particular for the optimal action $I^{*}$ and it follows
    \begin{align*}
        P\left(G_{max}\leq \hat{\Gamma}_{T,I^{*}}+c\sqrt{nT}\right)\geq 1-\delta
    \end{align*}
    Multiplying this inequality with $(1-2\gamma)$ and rearranging  \eqref{eq:simple_ineq}, we have with probability at least $1-\delta$
    \begin{align*}
        G_{max}-G&\leq \frac{2n(1-\gamma)}{\gamma}\log \frac{n}{k} + 2(1-\gamma)c\sqrt{nT}+c^{2} + 2\gamma G_{max}\\
         &\leq \frac{2n}{\gamma}\log \frac{n}{k} + 2c\sqrt{nT}+c^{2} + 2\gamma kT
    \end{align*}
    The second inequality follows from $1\geq 1-\gamma$ and $G_{max}\leq kT$. By definition of $\gamma$ and $c$, we conclude the proof with
    \begin{align*}
        R_{T} =  G_{max}-G \leq 2\sqrt{knT\log \frac{n}{\delta}} + 4\sqrt{knT\log \frac{n}{k}}+ k\log \frac{n}{\delta}.
    \end{align*}
    \end{proof}
Using the regret bound from Theorem \ref{th:EXP4.MP}, we derive the convergence of \eqref{eq:min_max_learning} on $\mathcal{S}_{n,k}$.
\begin{theorem}
Let $\delta>0$. Consider running $T\geq \max \{\log(n/\delta), n\log(n/k)/k \}$ rounds of Algorithm \ref{algo:online_bandit_k_set} with a choice of online learning algorithm for the $w$-player ensuring a worst case regret $R_T^w\leq C \sqrt{T}$ for some $C>0$. Further fix the parameters
\begin{align*}
    \eta=\frac{k\gamma}{2n},\quad \gamma=\sqrt{\frac{n\log(n/k)}{kT}},\quad c=\sqrt{k\log(n/\delta)}.
\end{align*}
Write $a_t$ the actions of the bandit and $w_t$ that of the online $w$-player.
Then with probability $1-\delta$, we have
\begin{equation*}
    \Delta(\Bar{w},\Bar{a}) \leq C\sqrt{\frac{1}{T}} + 2\sqrt{\frac{kn}{T} \log\left(\frac{n}{\delta}\right)} +4\sqrt{\frac{kn}{T}\log \left(\frac{n}{k}\right)}+ \frac{k}{T}\log\left( \frac{n}{\delta}\right),
\end{equation*}
where $\Bar{a}=\frac{1}{T}\sum_{t=1}^{T}a_t$ and $\Bar{w}=\frac{1}{T}\sum_{t=1}^{T}w_t$.
\end{theorem}

\begin{proof}[Proof of Theorem \ref{th:topk}]
Follows from Theorem \ref{thm:general}, when the $w$-player chooses OGD and the $p$-player chooses EXP4M. When $\ell(h(\cdot,x_{i}),y_{i})$ is $L$-Lipschitz, and the parameters $w$ lie in a $\ell_2$-Ball of size $B$, running OGD with step size of $\eta=\frac{B}{L\sqrt{2T}}$, we have \citep[Corollary 2.7]{shalev2011online},
\[
\frac{1}{T}\epsilon_w(T)\leq BL\sqrt{2/T}.
\]
From Theorem \ref{th:EXP4.MP} , we have that for parameters as defined in  \eqref{eq:def_hyper_parameters}
we have with probability $1-\delta$ that
\begin{align*}
    R_T\leq 2\sqrt{knT\log(n/\delta)}+4\sqrt{knT\log(n/k)} + k\log(n/\delta),
\end{align*}
which concludes the proof.
\end{proof}

\section{Generalizing the $k$-set: $\alpha$-set}\label{app:generalizing_k_set}

In this section, we suggest that efficient online-bandit strategies can be developed for \eqref{eq:min_max_learning} with sets $\mathcal{K}$ beyond the \eqref{eq:top_k_constraint}.
Consider a weight vector $\alpha \in [0, 1]^n$ satisfying $k \alpha \in \{1, \ldots, k\}^n$.
The following \emph{$\alpha$-set} is a generalization of the $k$-set
\begin{equation*}\label{eq:alpha-set}
    \cK_\alpha := \{p\in \R^n\mid 0 \leq p_i \leq  \alpha_i , \sum_{i=1}^n p_i = 1\}. \tag{$\alpha$-Set}
\end{equation*}
For $\cS_{n, k}$, the constraints $p_i\leq \frac{1}{k}$ uniformly limit the influence of any data point $x_i$, guaranteeing that the bandit takes into account at least $k$ data points in each iteration.
Conversely, the \eqref{eq:alpha-set} constraints provide an opportunity to treat data points heterogeneously or take into account prior information about the data set.
This flexibility might be helpful, for example, in the context of learning with non-standard aggregated losses \citep{shalev2016minimizing,fan2017learning} to distill external confidence score for the data points to be outliers.\\\\
However, we are not yet aware of learning problems involving \eqref{eq:min_max_learning} with an  \eqref{eq:alpha-set} that is not a \eqref{eq:top_k_constraint}.
Hence, we only outline the appealing properties of \eqref{eq:alpha-set} that makes it suitable to design an efficient online-bandit strategy.
Also, as opposed to the exhaustive treatment done in Section \ref{sec:average_top_k}, we do not delve upon an efficient sampling strategy similar to DepRound (Algorithm \ref{algo:depRound}) nor on an efficient combinatorial bandit adapted to \eqref{eq:alpha-set}.\\\\
In Section \ref{sec:char}, we first study the extremal structure of \eqref{eq:alpha-set} for some values of $\alpha\in[0,1]^n$ and show in Section \ref{sec:one_to_one}, that for these values this family satisfies Assumption \ref{assum:structure_K_alpha} similar in spirit to Assumption \ref{assum:structure_K}.

\subsection{Polyhedral Representation of $\text{Ext}(\cK_\alpha)$}\label{sec:char}
The result below characterizes the extremal structure of $\cK_\alpha$ for some values of $\alpha\in[0,1]^n$.

\begin{theorem}\label{th:extrem_K_alpha}
Consider $\alpha \in [0, 1]^n$ such that $k \alpha \in \{1, \ldots, k\}^n$, $p =(p_i)_{i=1,\cdots,n}\in \cK_\alpha$, and $I_p:=\{i \in [n]\mid p_i \in ]0, \alpha_i[\}$.
Then, $p\in \text{Ext} (\cK_\alpha)$ if and only if $|I_p| \leq 1$. 
\end{theorem}
\begin{proof}
Consider $p\in \cK_\alpha$ such that $|I_{p}|\leq 1$.

The set $\mathcal{K}_{\alpha}$ is defined by the following linear inequalities and equality constraints
\begin{equation*}
  \left\{
    \begin{split}
    &\sum_{i=1}^{n}{p_i}=1\\
    &p_i\geq 0 \text{ and } p_i\leq \alpha_i ~~\forall i\in[n].
    \end{split}
  \right.
\end{equation*}
In particular, since $|I_{p}|\leq 1$, there exists $n-1$ distinct indices $i\in[n]$ s.t. $p_i=0$ or $p_i=\alpha_i$.
Besides, since the inequality constraint $\sum_{i=1}^{n}p_i=1$ is satisfied, all the equality constraints defining $\mathcal{K}_{\alpha}$ are active for $p$ and there are $n$ active constraints that are linearly independent. Consequently, $p\in\mathcal{K}_{\alpha}$ is a \textit{basic feasible solution} for $\mathcal{K}_{\alpha}$ \citep[Definition 2.9.]{bertsimas1997introduction}.
Then, by \citep[Theorem 2.3]{bertsimas1997introduction}, we have $p\in\text{Ext} (\cK_\alpha)$, proving the first direction of the statement.\\\\
Conversely, consider $p\in \text{Ext} (\cK_\alpha)$ and suppose toward a contradiction that $|I_p| \geq 2$. 
Then, there exist $h,\ell \in [n]$, $h \neq \ell$, such that $p_h\in(0,\alpha_h)$ and $p_\ell\in (0,\alpha_\ell)$.
Let
\(\epsilon := \min\{p_{h}, \alpha_h - p_h, p_{\ell}, \alpha_\ell - p_\ell \}> 0\) 
and define the vectors $u, v\in \R^n$ as follows
\begin{equation*}
    u_i = \left.
  \begin{cases}
    p_i, & \text{for } i\not\in\{ {h}, \ell\} \\
    p_i + \epsilon, & \text{for } i = {h} \\
    p_i - \epsilon, & \text{for } i = \ell
  \end{cases}
  \right.
  \qquad \text{ and } \qquad
  v_i = \left.
  \begin{cases}
    p_i, & \text{for } i\not\in\{ {h}, \ell\} \\
    p_i - \epsilon, & \text{for } i = {h} \\
    p_i + \epsilon, & \text{for } i = \ell.
  \end{cases}
  \right.
\end{equation*}
By definition $\sum_{i=1}^{n}{u_i}=\sum_{i=1}^{n}{v_i}=\sum_{i=1}^{n}{p_i}=1$, $p_l \pm \epsilon \in [0,\alpha_l]$ and $p_h \pm\epsilon [0, \alpha_h]$ so that $u,v\in \cK_\alpha$.
However, by construction, $p = \frac{1}{2}(u + v)$ with $u\neq v$ which contradicts $p\not\in \text{Ext} (\cK_\alpha)$.
Hence, $|I_p|\leq 1$.
\end{proof}

\begin{corollary}\label{cor:int_vert}
Consider $\alpha \in [0, 1]^n$ such that $k \alpha \in \{1, \ldots, k\}^n$ and $p\in\text{Ext}(\cK_\alpha)$.
Then, we have $|\supp(p)| \leq k$ and $kp\in\{0,1,\ldots,k\}^n$.
\end{corollary}

\begin{proof}
Let $p\in\text{Ext}(\cK_\alpha)$. By Theorem \ref{th:extrem_K_alpha}, $|I_p| \leq 1$.
Assume by contradiction that $|\supp(p)| > k$ and write $J:=\{i\in[n] \text{ s.t. } p_i = \alpha_i\}$.
Since $|I_p|\leq 1$ and $|\supp(p)| > k$, we have $|J|\geq k$. Also, by definition of $\alpha$, we have $\alpha_i\geq \frac{1}{k}$. Hence,
\[
1 = \sum_{i=1}^{n}{p_i} = \sum_{i\in J}{\alpha_i} + \sum_{i\in I_p}{p_i} + 0 \geq \frac{|J|}{k} + \sum_{i\in I_p}{p_i}.
\]
Since $\frac{|J|}{k}\geq 1$, we have that $|I_p|=0$ and ultimately $|J|=k$ hence $|\text{supp}(p)|= k$ which is a contradiction.\\\\
Finally, let us show that $kp\in\{0,1,\ldots,k\}^n$.
We have $|I_p|\leq 1$.
For any $i\notin I_p$, we have $p_i=0$ or $p_i\alpha_i\in \{1/k,2/k,\cdots,1\}$.
If $|I_p|=0$, we hence have $kp\in\{0,1,\ldots,k\}^n$.
Otherwise, $I_p={i_0}$, and we have
\[
p_{i_0} + \sum_{i\in \text{supp}(p)\setminus \{i_0\}}{\alpha_i} = 1.
\]
Hence, with $k\alpha_i\in\mathbb{N}$ and $p_{i_0}\geq 0$, we have 
\[
k p_{i_0} = k - \sum_{i\in \text{supp}(p)\setminus \{i_0\}}{k \alpha_i},
\]
so that $k p_{i_0}$ is indeed an integer smaller than $k$.
\end{proof}

\subsection{Multiset representation of $\text{Ext}(\cK_\alpha)$}\label{sec:one_to_one}

Assumption~\ref{assum:structure_K} imposes the existence of an injection $\mathcal{T}:\text{Ext}(\mathcal{K}) \rightarrow \mathcal{P}^*\big([n]\big)$.
In Section \ref{sec:average_top_k}, with $\mathcal{K}=\mathcal{S}_{n,k}$, this assumption is verified since there is a bijection $\mathcal{T}:\text{Ext}(\mathcal{S}_{n,k}) \rightarrow \mathcal{P}_k\big([n]\big)$.
In particular, each vertex of the $k$-set corresponds to a sample \emph{without} replacement from $[n]$ of size $k$.
Here, a vertex of the $\alpha$-set rather correspond to a sample \emph{with} replacement from $[n]$ of size $k$ and Assumption~\ref{assum:structure_K} does not hold anymore.
Let us now introduce the set of multisets of cardinality at most $k$,

\begin{align*}
    \cP_{\leq k}([n]):= \Big\{\{(a_1, m(a_1)), \ldots, (a_j, m(a_j))\} \mid & j \in [k],  a_\ell \in [n] \text{ and } m(a_\ell) \in \N \text{ for all } \ell \in [j], \\
    & a_h \neq a_\ell \text{ for } h, \ell\in [j] \text{ and }  h \neq \ell, \sum_{i = 1}^jm(a_i) = k \Big\} \cup \{\emptyset\}.
\end{align*}
For a $I \in \cP_{\leq k}([n])$ and $(a, m(a)) \in I$, $a\in [n]$ refers to an index and $m(a) \in [k]$ to the multiplicity of $a$.
We can now adapt Assumption~\ref{assum:structure_K} to the case of $\cP_{\leq k}([n])$.

\begin{assumption}[Extremal Structure of $\mathcal{K}$]\label{assum:structure_K_alpha}
For the compact convex set $\mathcal{K}\subset\mathcal{S}_n$ there exists an injective function $\mathcal{T}:\text{Ext}(\mathcal{K_\alpha}) \rightarrow \mathcal{P}_{\leq k}\big([n]\big)$.
\end{assumption}

\noindent For some values of $\alpha\in [0,1]^n$, \eqref{eq:alpha-set} satisfies Assumption \ref{assum:structure_K_alpha}.

\begin{lemma}
Let $\alpha \in [0, 1]^n$ such that $k \alpha \in \{1, \ldots, k\}^n$. Then, $\cK_\alpha$ satisfies Assumption~\ref{assum:structure_K_alpha}.
\end{lemma}

\begin{proof}
Let $\alpha \in [0, 1]^n$ such that $k \alpha \in \{1, \ldots, k\}^n$ and $p \in \text{Ext}(\cK_\alpha)$.
Consider the function $\mathcal{T}$ defined via
\[p \mapsto \mathcal{T}(p):=\bigcup_{i \in \supp(p)}\{(i, k p_i)\}.\]
By Corollary~\ref{cor:int_vert}, $kp\in \{0, 1, \ldots, k\}^n$ and $|\text{supp}(p)|\leq k$.
Hence, $k p_i\in\mathbb{N}$ and since $\sum_{i\in\text{supp}(p)}{kp_i} = k$ we have that $\mathcal{T}(p)\in \cP_{\leq k}([n])$, \textit{i.e.},  $\text{Im}(\mathcal{T}(\text{Ext}(\mathcal{K}_{\alpha})))\subseteq \cP_{\leq k}([n])$.
It remains to prove that $\mathcal{T}$ is injective.
Suppose toward a contradiction that there exists $v\in \text{Ext}(\cK_\alpha)$ such that $p\neq v$ and $\mathcal{T}(p) = \mathcal{T}(v)$. By construction of $\mathcal{T}$, $p_i = v_i$ for all $i \in [n]$. Thus, $p = v$, a contradiction.
\end{proof}
\section{Parameters for Section \ref{sec:numerical_experiments}}\label{sec:params}
We call the $p$-player's learning rate $\eta_p$ and the $w$-player's learning rate $\eta_w$.
Further we write $N$ for the number of processed datapoints.
We provide the resulting learning rates for the specified data sets and algorithms in Table \ref{tab:topk}.
\begin{table}[h]
\centering
\begin{tabular}{@{}lcccc@{}}
\toprule
       & \multicolumn{2}{c}{Cancer} & \multicolumn{2}{c}{Boston} \\ \cmidrule(lr){2-3}\cmidrule(lr){4-5}
       & $\eta_w$     & $\eta_p$    & $\eta_w$     & $\eta_p$    \\ \midrule
S-AFL     & 3.51e-03     &   3.09e-05  &   5.22e-03  & 3.47e-05    \\
OL-FTRL   & 3.31e-02     & 1.95e-02    &  1.81e-01   &    7.04e-03 \\
OL-EXP.4M & 6.20e-03     & 2.43e-04    &  1.40e-03   &    2.53e-04 \\ \bottomrule
\end{tabular}
\caption{Learning rates for the $w$-player and the $p$-player used for experiments in Fig. \ref{fig:topk_cel} and Fig. \ref{fig:topk_mse} with a number of rounds $T$ corresponding to processing $10^7$ data points with $k=20$.}
\label{tab:topk}
\end{table}
\subsection{OL-EXP.4M}
We set the parameters for EXP.4M as called for in \citep{vural2019minimax}, \textit{i.e.},
\begin{equation*}
    \eta_p=\frac{k\gamma}{2n},\quad \gamma=\sqrt{\frac{n\log(n/k)}{kT}},\quad c=\sqrt{k\log(n/\delta)},
\end{equation*}
and $\eta_w=B\sqrt{\frac{2}{T}}$ \citep[Corollary 2.7]{shalev2011online}. We set the number of iterations as $T=\frac{N}{k}$ as each iteration requires sampling $k$ datapoints.
\subsection{OL-FTRL}
We set $\eta_w=B\sqrt{\frac{2}{T}}$ \citep[Corollary 2.7]{shalev2011online} and the number of iterations as $T=\frac{N}{n}$ as each iteration requires processing all $n$ datapoints.
In order to obtain $\eta_p$, we provide the regret bound for FTRL on $\mathcal{S}_{n,k}$ which is just a slight variation of the classical regret bound of FTRL on the simplex \citep[Cor. 2.14]{shalev2011online}.
\begin{theorem}
Let $R(u)=\frac{1}{\eta}\sum_{i=1}^n u_i\log(u_i)$ then running FTRL on $u\in \mathcal{S}_{n,k}$ for $T$ rounds with $\eta = \sqrt{\frac{\log(n/k)}{T}}$ guarantees
\begin{equation*}
    R_T\leq 2\sqrt{T\log \left(\frac{n}{k}\right)}.
\end{equation*}
\end{theorem}

\begin{proof}
First note that $R$ is $\frac{1}{\eta}$-strongly convex with respect to the 1-norm. Then by \citep[Theorem 2.15]{shalev2007convex} it holds for any $u\in \mathcal{S}_{n,k}$,
\begin{equation}\label{eq:FTRL_reg}
    R_T(u)\leq R(u) - \min_{v\in \mathcal{S}_{n,k}} R(v) + \eta T.
\end{equation}
Note that $R$ is the negative entropy multiplied with the septsize $\eta$, \textit{i.e.},$R(u) = -\frac{1}{\eta} H(u)$ hence 
\begin{equation*}
     -\min_{v\in \mathcal{S}_{n,k}} R(v) = -\min_{v\in \mathcal{S}_{n,k}} \left[ -\frac{1}{\eta} H(v)\right] = \max_{v\in \mathcal{S}_{n,k}} \frac{1}{\eta} H(v) = \frac{1}{\eta}\log(n).
\end{equation*}
The last equality holds since the entropy is maximized by the uniform distribution.
We now go on to upper bound $R(u)$, which corresponds to finding $u\in \mathcal{S}_{n,k}$ with minimum entropy. 
The minimal entropy of $u\in \mathcal{S}_{n,k}$ is achieved by 
\begin{equation*}
    u^* = (\;\underbrace{1/k,\ldots,1/k}_{k}\;,\; \underbrace{0, \ldots, 0}_{n-k}\;)^T
\end{equation*}
and $H(u^*)=-\log(k)$. Plugging these two terms into \eqref{eq:FTRL_reg} we obtain
\begin{equation*}
    R_T(u)\leq \frac{1}{\eta}\log\left(\frac{n}{k}\right) + \eta T.
\end{equation*}
setting $\eta = \sqrt{\frac{\log(n/k)}{T}}$ concludes the proof.
\end{proof}

\subsection{S-AFL}
We set the number of iterations as $T=\frac{N}{2k}$ as each iteration requires sampling $k$ datapoints for the $w$-player and the $p$-player respectively.
In order to compute the theoretical stepsizes following \citep[Theorem 2]{mohri2019agnostic}, we need to specify the function $f$ to be learned as they depend on the $L$-smoothness of $f$. 
We give an overview to the derived quantities in Table \ref{tab:AFL}.
First, recall the theoretical stepzsizes from \citep{mohri2019agnostic},
\begin{align*}
    \eta_w= \frac{2R_\mathcal{W}}{\sqrt{T(\sigma^2_w + G^2_w)}},\quad  \eta_\lambda = \frac{2R_\Lambda}{\sqrt{T(\sigma_\lambda^2 + G_\lambda^2)}},
\end{align*}
since $L_i(w)\leq 1$ by assumption.
We write $\delta$ for the gradient estimate and $\nabla$ for the true gradient.
Further we write the define $g(w,p)=\langle L(w);p\rangle$.
\[\Vert \nabla_p g(w,p)\Vert^2  = \Vert L(w)\Vert^2= \sum_{i=1}^n L_i(w)^2\leq n.\]
We use the \textit{Weighted Stochastic Gradient} approach, hence  $\sigma_w^2=\sigma_I^2 + \sigma_O^2$. 
Since each subgroup consists of only one data point, we have $\sigma_I=0$.
Now we bound the $\sigma_w^2=\sigma_O^2$,
\[\sigma_w^2=\sum_{i=1}^n\lambda_i\left\Vert\delta_w L_i(w) - \sum_{j=1}^n\lambda_j\nabla_w L_j(w)\right\Vert^2\leq \sum_{i=1}^n\lambda_i\left(\Vert\delta_w L_i(w)\Vert + \Vert\sum_{j=1}^n\lambda_j\nabla_w L_j(w)\Vert\right)^2\leq (2G_w)^2.\]
\begin{table}[h]
\centering
\begin{tabular}{@{}lcccccc@{}}
\toprule
         & $R_\mathcal{W}$ & $R_\Lambda$ & $\sigma_\lambda^2$ & $\sigma_w^2$ & $G_\lambda^2$ & $G_w^2$ \\ \midrule
AFL (MSE) &   $B$    & 1 &           $n^2/k$ &     $16 (B+1)^2$       &  $n$      &    $4(B+1)^2$    \\
AFL (CEL) &   $B$    &  1     &     $n^2/k$        &     $8 (d^2 +1)$       &   $n$     &     $2d^2+2$   \\ \bottomrule
\end{tabular}
\caption{Relevant quantities to compute stepsizes for Stochastic AFL variables for linear regression (MSE) and logistic regression (CEL)}
\label{tab:AFL}
\end{table}
\paragraph{Classification.}
Let us consider logistic softmax regression with CEL. 
The classifier is defined as 
\[
f(w,x) = \sigma(\Theta x +b), \quad \Theta \in \mathbb{R}^{C\times d},\;b\in \mathbb{R}^C,\;x\in \mathbb{R}^d,\;y \in \mathbb{R}^C 
\]
where $\sigma$ is the softmax function, $C$ is the number of classes and $d$ is the feature dimension.
First note
\[
\nabla_w g(w,p) = \sum_{i=1}^n p_i \nabla_w L_i(w), \quad L_i(w) = \text{CEL}\big(f(w,x_i),y_i\big).
\]
In order to derive an upper bound to $\nabla_w g(w,p)$ we choose weight $p_i=1$ for the largest single derivate which reduces the problem to upper bounding $\nabla_w L_i(w)$. We can write this derivate as
\begin{align*}
    \frac{\partial L(w)}{\partial \Theta_{i,j}} = \underbrace{x_i}_{\in [0,1] \text{(by def.)}} \underbrace{(f_j(x) -y_j)}_{\in[-1,1]}, \quad \frac{\partial L(w)}{\partial b_{j}} = (f_j(x) -y_j).
\end{align*}
Using these identities, we obtain the following upper bound,
\begin{align*}
    G_w^2\leq \Vert\nabla_{\Theta,b}L(\Theta,b)\Vert_F^2 = \sum_{i=1}^C(d(f_i-y_i))^2 + \sum_{i=1}^C(f_i-y_i)^2
    \leq 2d^2+2.
\end{align*}

\paragraph{Regression.}
We now repeat the analysis for linear regression with MSE.
The regressor is defined as
\[
f(w,x) = \theta^Tx+b, \quad x\in [0,1]^{d},\;y \in [0,1],\;\theta\in \mathbb{R}^d,\;b\in \mathbb{R}.
\]
This leads to the following quantities
\begin{align*}
    \nabla_w g(w,p) &= \sum_{i=1}^n p_i \nabla_w L_i(w), &L_i(w) &= \text{MSE}\big(f(w,x_i),y_i\big) = \Vert f(w,x_i)-y_i \Vert_2^2,\\
    \nabla_w L_i(w) &= 2(f(x_i)-y_i)x, & \nabla_b L(w,x_i) &= 2(f(x_i)-y_i).
\end{align*}
We can now find an upper bound for $G_w^2$,
\begin{align*}
    G_w^2 = \Vert2(f(x)-y)x \Vert_2^2 = \sum_i^n 4(f(x)-y)^2x_i^2 \leq 4 \left(\sum_i^n w_i+b-y\right)^2\leq 4(B+1)^2.
\end{align*}
The last inequality follows from setting $y=1$, $b=B$ and $w=0$.

\end{document}